\documentclass{article}

\usepackage{arxiv}

\usepackage{mathtools} % amsmath with fixes and additions
\usepackage{booktabs} % commands to create good-looking tables
\usepackage{tikz} % nice language for creating drawings and diagrams

\usepackage{latexsym,amsmath,amsfonts,amssymb,amsthm} 
\usepackage{url}
\usepackage{graphicx} 
\usepackage{epstopdf}
\usepackage{subfigure}
\usepackage{graphics, color}
\usepackage{hyperref}
\usepackage{natbib}
\usepackage{stmaryrd}
\usepackage{enumitem,nicefrac}
\usepackage{stmaryrd}
\usepackage{enumitem,nicefrac}

\usepackage[absolute,overlay]{textpos}

\usepackage{bbm}
\usepackage{algorithm,algorithmic}
\usepackage{cleveref}

\newcommand{\IND}[1]{\mathbbm{1}_{\llbracket{#1}\rrbracket}}

\newcommand{\Algo}[1]{\textsc{#1}}

\newcommand{\prob}{\mathbb{P}} 
\newcommand{\exptd}{\mathbb{E}}

\newcommand{\argmax}{\operatornamewithlimits{argmax}}

\newcommand\R{\mathbb{R}}   % for the real numbers
\newcommand\N{\mathbb{N}}   % for the natural numbers
\newcommand{\KL}{\mathrm{KL}}

\newcommand{\tern}{P^{(i,j)}}
\newcommand{\terntilde}{\widetilde{P}^{(i,j)}}

\theoremstyle{plain}
\newtheorem{theorem}{Theorem}[section]

\newtheorem{lemma}[theorem]{Lemma}
\newtheorem{corollary}[theorem]{Corollary}
\theoremstyle{definition}
\newtheorem{definition}[theorem]{Definition}

\theoremstyle{remark}

\title{Identifying Copeland Winners in Dueling Bandits with Indifferences}

\author{
    Viktor Bengs$^{b,c},$ Bj\"orn Haddenhorst$^{a},$
	Eyke H\"ullermeier${}^{b,c}$\\
	${}^{a}$Department of Computer Science, Paderborn University, Germany\\
	${}^{b}$Institute of Informatics, University of Munich (LMU), Germany\\
	${}^{c}$Munich Center for Machine Learning, Germany\\
	\texttt{viktor.bengs@lmu.de, bjoernha@mail.upb.de, eyke@lmu.de} 
}

\begin{document}
\maketitle

\begin{abstract}
	We consider the task of identifying the Copeland winner(s) in a dueling bandits problem with ternary feedback.
	This is an underexplored but practically relevant variant of the conventional dueling bandits problem, in which, in addition to strict preference between two arms, one may observe feedback in the form of an indifference.
	We provide a lower bound on the sample complexity for any learning algorithm finding the Copeland winner(s) with a fixed error probability.
	Moreover, we propose \Algo{POCOWISTA}, an algorithm with a sample complexity that almost matches this lower bound, and which shows excellent empirical performance, even for the conventional dueling bandits problem.
	For the case where the preference probabilities satisfy a specific type of stochastic transitivity, we provide a refined version with an improved worst case sample complexity.
\end{abstract}

\section{INTRODUCTION}
Dueling bandits \citep{YuJo09} or the more general problem class of preference-based bandits \citep{bengs2021preference} is a practically relevant variant of the standard reward-based multi-armed bandits \citep{LaSz18}, in which a learner seeks to find in a sequential decision making process an optimal arm (choice alternative) by selecting two (or more) arms as its action and obtaining feedback in the form of a noisy preference over the selected arms.
The setting is motivated by a broad range of applications, where no numerical rewards for the actions are obtained and only comparisons of arms (choice alternatives) are possible as actions.
In information retrieval systems with human preference judgments \citep{clarke2021assessing}, for example, humans choose their most preferred choice alternative among the two (or more) retrieved choice alternatives (e.g., text passages, movies, etc.).
Another example is the analysis of voting behavior, in which voters express their preferences over pairs of political parties or candidates \citep{brady1989nature}.
The rationale for these types of applications is that humans are generally better at coherently expressing their preference for two choice alternatives than at reliably assessing those two on a numerical scale \citep{CarteretteBCD08,li2021roial}.

Even though there is a large body of literature on dueling or preference-based bandits (see \cite{SuZoHoYu18,bengs2021preference}) covering various variants or aspects of the initial setting, little attention has been paid to the variant, in which a learner might observe an indifference between the selected arms for comparison as the explicit feedback.
In practice, however, this type of feedback is quite common, especially when the preference feedback is provided by a human.
In the human preference judgment example above, the human might be indifferent between the two retrieved choice alternatives, as the two are considered equally good/mediocre/bad, so neither is chosen.
Similarly, voters might be indifferent between two political parties or candidates, or two athletes resp.\ sport teams competing against each other might draw.

In several areas of preference-based learning, appropriate extensions of models or methods have been considered to appropriately incorporate such indifferences. Notable examples are the recent extensions of established probabilistic ranking models \citep{firth2019davidson,DBLP:journals/cstat/TurnerEFK20,henderson2022modelling}, the field of partial label ranking \citep{DBLP:journals/ijis/AlfaroAG21,DBLP:journals/entropy/RodrigoAAG21,DBLP:conf/pgm/AlfaroA022,DBLP:journals/ijar/AlfaroAG23,DBLP:journals/pr/AlfaroAG23}, an extension of the established label ranking problem, or preference-based Bayesian optimization \citep{dewancker2018sequential,nguyen2021top}.
However, the field of preference-based bandits lags behind in this regard, as the only work appears to be \cite{GaUrCl15} that considers an adversarial learning scenario for regret minimization.
%Despite the practical relevance, the only work in this regard seems to be  \cite{GaUrCl15}, which considers an adversarial learning scenario for regret minimization.
%, which is, however, known to be overly pessimistic for practical usage.

Motivated by this gap in the literature on preference-based bandits, we consider the \emph{stochastic} dueling bandits problem with indifferences (or ternary feedback) for the task of identifying an optimal arm as quickly as possible, i.e., with as few queried feedback observations as possible.
Similarly as in the conventional dueling bandits problem with binary feedback (i.e., strict preferences) specifying the notion of optimality of an arm raises some issues.
Indeed, the most natural notion of an optimal arm would be an arm, which is non-dominated by any other arm in terms of the probability of being strictly preferred or indifferent.
This notion corresponds to the Condorcet winner (CW) in the conventional setting, where the probability of observing indifferences is zero.
However, it is well known that such an arm may not exist in general in the conventional setting, an issue obviously shared by the adopted CW for our considered setting.
In contrast to the conventional CW, the adopted CW does not even guarantee uniqueness of the optimal arm\footnote{As an example, consider the case of three arms, where each arm has a probability of 1/2 of being preferred over or indifferent to another arm, respectively.}.

The non-existence issue of the CW has led several authors to consider alternative notions for the optimality of arms guaranteed to exist in any case.
Most of them have their roots in tournament solutions used in social choice and voting theory \citep{brandt2015bounds,brandt16tournaments} or game theory \cite{owen2013game}.
One popular alternative is the Copeland set \citep{Co51} defined as the set of choice alternatives (arms) with the highest Copeland score.
In the absence of indifferent preferences, the Copeland score of a choice alternative is the number of choice alternatives it dominates in terms of the pairwise preference probability.
In settings like ours, i.e., where indifferences might be present, the Copeland score of a choice alternative is the sum of (i) the number of choice alternatives it \emph{strictly} dominates, and (ii) half the number of choice alternatives it is most likely indifferent to.
Again, both definitions coincide for the conventional setting, where the probability of observing indifference is zero.

As the term itself suggests, there can be several arms in the Copeland set, each of which is called a Copeland winner (COWI).
Despite this non-uniqueness issue, the main advantages of considering the Copeland set are that (i) it is guaranteed to exist and (ii) that it consists only of the Condorcet winner(s) in case of its (their) existence\footnote{In the setting with indifferences, there might be multiple Condorcet winners.}.
Moreover, the majority of alternative optimality notions from tournament solutions are in fact supersets of the Copeland set (see \cite{RaRaAg16}).

%\textbf{Outline and Contributions}
\paragraph{
	\textbf{Outline and Contributions}}
	\begin{enumerate}
%		[noitemsep,topsep=0pt,leftmargin=5mm]
  
\item \textbf{Introduction of the problem (Section \ref{subsec:related_work}):} We are the first to consider the problem of finding a Copeland winner in a dueling bandits problem with possible indifference observations. This problem variant is of practical relevance, especially for applications involving human feedback.

\item \textbf{Lower bounds (Section \ref{subsec:lower_bounds}):} We provide lower bounds on the sample complexity for any learning algorithm to find a Copeland winner with a fixed confidence in this problem variant. 
Our lower bounds imply as a special case a long-missing lower bound for the conventional dueling bandits problem with only strict preference observations.

\item \textbf{A practically useful (Section \ref{sec:experiments}) and near-optimal algorithm (Section \ref{sec:learning_algorithm}):} We construct a learning algorithm, \Algo{POCOWISTA}, which selects pairs of arms in a challenge-tournament-like fashion and exploits prior-posterior-ratio (PPR) martingale confidence sequences for determining the Copeland scores of the underlying arms in an asymptotically optimal way.
In numerical simulations, we demonstrate that it performs quite well even for conventional dueling bandits.

\item \textbf{Relaxing quadratic dependency (Section \ref{sec:transitivity}):} We show that in the case of an underlying transitivity of the preference relations, the update formula of \Algo{POCOWISTA} can be modified (called  \Algo{TRA-POCOWISTA}) such that the quadratic dependency w.r.t.\ the number of arms $n$ of the worst-case sample complexity can be relaxed to a log-linear dependency. 
\end{enumerate}

\section{PROBLEM FORMULATION} \label{sec:problem}

We consider a set $\mathcal{A}$ of $n\in \N_{\geq 2}$ available choice alternatives that we refer to as arms and simply identify them by their index: $\mathcal{A}=\{1,\ldots,n\}.$
The learning process consists of consequential rounds, in which the learner performs an action leading to some feedback for its action.
More precisely, the learner's action in round $t$ corresponds to choosing a pair of arms $(i_t,j_t) \in \mathcal{A} \times \mathcal{A},$ for which it observes noisy feedback $o_t$ with three possible realizations: 
%(i) $i_t \succ j_t,$ i.e., arm $i_t$ is strictly preferred over arm $j_t,$ (ii) $i_t \prec j_t,$ i.e., arm $j_t$ is strictly preferred over arm $i_t,$ or (iii) $i_t \cong j_t,$ i.e., neither arm $i_t$ is strictly preferred over arm $j_t$ or the opposite (indifference between arms $i_t$ and $j_t$).
%
\begin{itemize}
	[noitemsep,topsep=0pt,leftmargin=4mm]
	\item $i_t \succ j_t,$ i.e., arm $i_t$ is strictly preferred over arm $j_t,$
	\item $i_t \prec j_t,$ i.e., arm $j_t$ is strictly preferred over arm $i_t,$
	\item $i_t \cong j_t,$ i.e., neither arm $i_t$ is strictly preferred over arm $j_t$ or the opposite (indifference between arms $i_t$ and $j_t$).
\end{itemize}

Each of the three possible explicit observations is determined by one of the following matrices  $P_{\succ},P_{\prec},P_{\cong} \in [0,1]^{n\times n.}$  
%We assume that the feedback is noisy and that each of the three possible explicit observations is determined by one of the following matrices  $P_{\succ},P_{\prec},P_{\cong} \in [0,1]^{n\times n.}$
%
Here, the $(i,j)$-th entry of $P_{\succ}$ \big(or $P_{\prec}$\big) denoted by $P_{\succ}^{(i,j)}$ \big(or $P_{\prec}^{(i,j)}$\big) specifies the probability of observing a strict preference of $i$ over $j$ (or $j$ over $i$), while the $(i,j)$-th entry of $P_{\cong}$ denoted by $P_{\cong}^{(i,j)}$  specifies the probability of observing an indifference between $i$ and $j.$
Apparently, it holds that $P_{\succ}^{(i,j)}+P_{\cong}^{(i,j)}+P_{\prec}^{(i,j)} = 1$ for any $i,j\in \mathcal{A},$ and consequently any dueling bandits problem with indifferences is uniquely determined by one of its strict preference probability matrices, since $P_{\succ}^{(j,i)} = P_{\prec}^{(i,j)}.$
The Copeland score of arm $i\in \mathcal{A}$ is
\begin{align}\label{def:Copeland_score}
	\begin{split}
		\text{CP}(i) &= \sum\nolimits_{j \neq i} \IND{ P_{\succ}^{(i,j)} > \max\{ P_{\prec}^{(i,j)},P_{\cong}^{(i,j)} \} } 
		 + \frac12  \sum\nolimits_{j \neq i} \IND{ P_{\cong}^{(i,j)}  > \max\{ P_{\succ}^{(i,j)} , P_{\prec}^{(i,j)} \} },
	\end{split}	
\end{align}
where $\IND{\cdot}$ denotes the indicator function.
Thus, an arm gets a score of one for each arm it dominates and half a point for each arm it is indifferent to.  
The Copeland set (or the set of Copeland winners) consists of all arms with maximal Copeland score and is denoted by 
\begin{align}\label{def:Copeland_set}
	\mathcal{C}= \{ i \in \mathcal{A} \, | \, \text{CP}(i) = \max\nolimits_j \text{CP}(j) \}.
\end{align}
The goal of the learner is to find an element of the Copeland set, i.e., a Copeland winner (COWI), by performing as few actions as possible.
To this end, the learner may decide to stop the learning process at some round $\tau$ and output an arm $\hat{i} \in \mathcal{A}$ deemed to be a 
COWI.
%Copeland winner.
%
Because of the stochasticity of the observed feedback, the learner can make mistakes, so any reasonable learner should meet a theoretical guarantee that its output is correct.
Thus, if $\delta \in(0,1)$ is the desired bound on the error probability, it should hold that $\prob(\hat{i} \notin \mathcal{C}) \leq \delta.$
Additionally, the learner's stopping time $\tau$ should be as small as possible (in expectation or with high probability), while still guaranteeing the latter error probability bound.

\subsection{RELATED WORK} \label{subsec:related_work}

The conventional dueling bandits problem, i.e., without indifferences, has been introduced as a practical variant of the classical multi-armed bandit (MAB) problem in \cite{YuJo09}.
Initially, the problem has been studied intensively for the task of regret minimization under the assumption of an existing CW \citep{YuBrKlJo12,ZoWhMuDe14,ZoWhDe15,KoHoKaNa15,agarwal2020choice} to specify a target arm.
Due to the non-existence issue of the CW, several works have considered alternative optimality notions for the target arm such as the COWI \citep{ZoKaWhDe15,koHoNa16,WuLi16} or other tournament solutions \citep{RaRaAg16}.

Similar to the classical MAB problem, there has been also much research interest in the pure exploration task in the conventional dueling bandits problem, where the goal is to identify the target arm as quickly as possible.
Again, the majority of works have used the CW to specify the target arm \citep{Ka16,MoSuEl17,ren2019sample,ren20a} or a generalized variant for multi-dueling settings \citep{haddenhorst2021identification,brandt2022finding}.
Although the identification of a COWI has been in fact studied before the aforementioned works \citep{BuSzWeChHu13,UrClFeNa13,BuSzHu14} none of these considered the case where indifferences are observed as explicit feedback.

%Although the identification of a Copeland winner has interestingly been studied previously in a number of papers , none of the above works have considered the case of observing indifferences as feedback.  

In practical use cases, observing indifferences as pairwise preference feedback plays an important role, as has been highlighted by a number of papers with applications ranging from sports \citep{tiwisina2019probabilistic}, medicine \cite{li2021roial}, crowdsourcing tasks \citep{DBLP:conf/cikm/AsudehZHLZ15,clarke2021assessing} to information retrieval \citep{yan2022human}.
%In contrast, indifferences play an important role in practical use cases, when dealing with pairwise preference feedback.
%
%This has been highlighted in a number of papers with applications ranging from sports \citep{tiwisina2019probabilistic}, medicine \cite{li2021roial}, crowdsourcing tasks \citep{clarke2021assessing} to information retrieval \citep{yan2022human}.
%
In addition, there is recent work in preference-based Bayesian optimization \citep{gonzalez2017preferential} that incorporates indifference feedback into the optimization procedure \citep{dewancker2018sequential,nguyen2021top}.

In the dueling resp.\ preference-based bandit literature, however, the only work which has considered indifferences is \cite{GaUrCl15}, which addresses an adversarial learning scenario for the task of regret minimization.
To this end, the \Algo{Sparring} algorithm \citep{AiKaJo14}, which uses two bandit algorithms for the classic MAB setting (each of which selects one arm of the pair to be dueled), is used with two instantiations of \Algo{Exp3} \citep{AuCeFrSc02} suitably modified to account for preference feedback.

\subsection{LOWER BOUNDS}\label{subsec:lower_bounds}

%Recall that any dueling bandits problem with indifferences is determined by the underlying indifference probability matrix $P_{\cong}$ and the strict preference probability matrix $P_{\succ}.$
%
For convenience, write $\mathbf{P} = ((P_{\succ}^{(i,j)},P_{\cong}^{(i,j)}, P_{\prec}^{(i,j)}))_{i<j}$, and denote by $P_{(1)}^{(i,j)}, P_{(2)}^{(i,j)}, P_{(3)}^{(i,j)}$  the order statistics of $P_{\succ}^{(i,j)}, P_{\cong}^{(i,j)}$ and $P_{\prec}^{(i,j)}$.
For any learning algorithm $\Algo{A}$ for the dueling bandits problem with indifferences let $\tau^{\Algo{A}}(\mathbf{P})$ denote its number of samples, when started on a dueling bandits problem characterized by $\mathbf{P}$. If $\mathbf{P}$ is clear from the context, we simply write $\tau^{\Algo{A}}$.
Furthermore, we denote by $$I(j) =\{ i\in \mathcal{A}\setminus\{j\} \, | \, P_{\cong}^{(i,j)} > \max\{ P_{\prec}^{(i,j)},P_{\succ}^{(i,j)} \} \}$$ the set of all indifferent arms, and by $$L(j) =\{ i\in \mathcal{A}\setminus\{j\} \, | \, P_{\succ}^{(i,j)} > \max\{ P_{\prec}^{(i,j)},P_{\cong}^{(i,j)} \} \}$$  the set of all superior arms to some arm $j\in \mathcal{A}.$
We write $\mathcal{C}(\mathbf{P})$ for the Copeland set in \eqref{def:Copeland_set} to highlight its dependence on the underlying dueling bandits problem with indifferences.
% $P_{(1)}^{(i,j)}\geq P_{(2)}^{(i,j)} \geq P_{(3)}^{(i,j)}$ is the order statistic of $P_{\succ}^{(i,j)}, P_{\cong}^{(i,j)}$ and $P_{\prec}^{(i,j)}.$ 
%
Define $\mathrm{KL}((p_1,p_2,p_3),(q_1,q_2,q_3))$  for the Kullback-Leibler divergence between two categorical random variables with parameters $(p_1,p_2,p_3)$ and $(q_1,q_2,q_3),$ while we use the common notation $\mathrm{kl}(p,q)$ for the Kullback-Leibler divergence between two categorical random variables with parameters $(p,1-p)$ and $(q,1-q),$ i.e., two Bernoulli distributions with success parameters $p$ and $q.$
%Define $\mathrm{KL}((p_1,p_2,p_3),(q_1,q_2,q_3))$ resp.\ $\mathrm{kl}(p,q)$ for the Kullback-Leibler divergence of two categorical random variables with parameters $(p_1,p_2,p_3)$ and $(q_1,q_2,q_3)$ resp.\ $(p,1-p)$ and $(q,1-q)$.
%
Finally, if $\mathbf{P}$ is fixed, let $d_j = \max_i \text{CP}(i) - \text{CP}(j)$ be the difference between the largest Copeland score and arm $j$'s Copeland score.
%Furthermore, we write $\text{CP}(i;P_{\succ},P_{\cong})$ for the Copeland score of arm $i\in\mathcal{A}$ in \eqref{def:Copeland_set} to highlight its dependence on the underlying dueling bandits problem with indifferences.
%
%In the same spirit we write $\mathcal{C}(P_{\succ})$ for the Copeland set in \eqref{def:Copeland_set}.

% \begin{theorem} \label{theorem:lower_bound}
% If $\Algo{A}$ is an algorithm that correctly identifies a COWI with confidence $1-\delta$ for any $P_\succ$, then for all $P_\succ$ with $ \mathcal{C}(P_{\succ}) = \{i^{\ast}\}$ (for some $i^\ast$, i.e., a unique COWI) we have 
%     \begin{align*}
%         \exptd [\tau^{\Algo{A}}(P_\succ)] \geq \ln \frac{1}{2.4\delta}\sum\nolimits_{j\not= i^\ast} C_j\min_{ k \in L(j) \cup I(j)} \frac{1}{\Delta^{2}^{(j,k)}},
%     \end{align*}
%     where 
%     \begin{align*}
%         C_j &\coloneqq \max_{(i,l) \in \Psi(j)} \frac{\binom{|I(j)|}{i} \binom{|L(j)|}{l}}{\binom{|I(j)|-1}{i-1} \binom{|L(j)|}{l} \IND{i\geq 1} + \binom{|I(j)|}{i} \binom{|L(j)|-1}{l-1} \IND{l\geq 1}},\\
%         \Psi(j) &\coloneqq  \left\{ (i,l) \, : \, i \in \{0,\dots,|I(j)|\}, l\in \{0,\dots,|L(j)|\} \text{ and } i+2l \geq 2d_j+1 \right\}, \\
%         \Delta(j,k)  &\coloneqq | P_{(1)}^{(j,k)} - P_{(2)}^{(j,k)} | 
%     \end{align*}
%     and $P_{(1)}^{(j,k)}\geq P_{(2)}^{(j,k)} \geq P_{(3)}^{(j,k)}$ is the order statistic of $P_{\succ}^{(j,k)}, P_{\cong}^{(j,k)}$ and $P_{\prec}^{(j,k)}.$
% \end{theorem}

\begin{theorem}\label{theorem:lower_bound_wo_ind_detailed_main_part}
	If $\Algo{A}$ correctly identifies the COWI with confidence $1-\delta,$ then 
	\begin{align} \label{simple_indiff_lower_bound}
		 \mathbb{E} [\tau^{\mathrm{A}}(P_\succ)] \geq \ln \frac{1}{2.4\delta }  \sum\nolimits_{j\in \mathcal{A} \setminus \{i^\ast\}} C_j \min\limits_{ k \in L(j) \cup I(j)} \frac{1}{D_{j,k}(\mathbf{P})}, 
	\end{align}
	where
 	\begin{itemize}[noitemsep,topsep=0pt,leftmargin=5.3mm] 	
		\item[(i)]  $C_j = \max \Big\{ \frac{|L(j)|}{d_j +1 } \IND{|L(j)| \geq d_j + 1} , \frac{|L(j)|-1}{|L(j)| + d_j - 2} \IND{i^\ast \in L(j)} \Big\}$ and $D_{j,k}(\mathbf{P}) = \mathrm{kl}( P_{\succ}^{(j,k)}, 1- P_{\succ}^{(j,k)})$ for all $\mathbf{P}$ without indifferences (i.e., $P_{\cong} = \mathbf{0} \in [0,1]^{n\times n}$).
		\item[(ii)]  $\Psi(j) = \left\{ (i,l) \, : \, i \in \{0,\dots,|I(j)|\}, l\in \{0,\dots,|L(j)|\},  i+2l \geq 2d_j+1 \right\}$ and 
  		\begin{align*}
			C_j &\coloneqq \max_{(i,l) \in \Psi(j)} \frac{\binom{|I(j)|}{i} \binom{|L(j)|}{l}}{\binom{|I(j)|-1}{i-1} \binom{|L(j)|}{l} \IND{i\geq 1} + \binom{|I(j)|}{i} \binom{|L(j)|-1}{l-1} \IND{l\geq 1}},\\
			D_{j,k}(\mathbf{P}) &\coloneqq \max \big\{ 
			\mathrm{KL}( ( P_{\succ}^{(j,k)},P_{\cong}^{(j,k)},P_{\prec}^{(j,k)}), ( P_{\cong}^{(j,k)},P_{\succ}^{(j,k)},P_{\prec}^{(j,k)} ) ), 
			\\
			&\phantom{abcdefghi} \mathrm{KL}( ( P_{\succ}^{(j,k)},P_{\cong}^{(j,k)},P_{\prec}^{(j,k)}), ( P_{\prec}^{(j,k)},P_{\cong}^{(j,k)},P_{\succ}^{(j,k)}) )
			\big\}.
		\end{align*}
		%$D_{j,k}(\mathbf{P}) = \max( \mathrm{KL}( ( P(j \succ k), P(j \cong k), P(j \prec k) ), ( P(j \cong k), P(j \succ k), P(j \prec k) ), \mathrm{KL}(( P(j \succ k), P(j \cong k), P(j \prec k) ), ( P(j \prec k), P(j \cong k), P(j \succ k) ))).$
	for any $\mathbf{P}$ (possibly with indifferences)  fulfilling $\min_{j,k} \min\{P_{\succ}^{(j,k)},P_{\cong}^{(j,k)},P_{\prec}^{(j,k)}\} > 0.$ 
	\end{itemize}

\end{theorem} 
A lower bound that would be a bit more natural to interpret would be 
%
%$$ \ln \frac{1}{2.4\delta} \sum_{j\in \mathcal{A} \setminus \{i^\ast\}}  \sum_{z \in L(j) } \frac{1}{\kappa_{j,z}(\mathbf{P})}, $$
%%
%for (i), and 
%
\begin{align} \label{eq_natural_lower_bound}
	 \ln \frac{1}{2.4\delta} \sum\limits_{j \in \mathcal{A} \setminus \{ i^\ast\}} \sum_{k \in L(j) \cup I(j) } \frac{1}{D_{j,k}(\mathbf{P})},
\end{align}
%
%for (ii), which 
that resembles the well-known lower bounds in the standard multi-armed bandits literature (e.g., Theorem 4 in \cite{kaufmann2016complexity}. 
However, note that we get a lower bound for the latter as follows:
%$$  \sum_{z \in L(j) } \frac{1}{\kappa_{j,z}(\mathbf{P})} \geq \min_{z\in L(j)} \frac{1}{\kappa_{j,z}(\mathbf{P})} \cdot |L(j)| $$
%and 
$$   \sum_{k \in L(j) \cup I(j) } \frac{1}{D_{j,k}(\mathbf{P})} \geq \min_{k\in L(j) \cup I(j)}  \frac{1}{D_{j,k}(\mathbf{P})} \cdot |L(j) \cup I(j)|. $$
In light of this, the $C_j$ terms on the right-hand side of Theorem \ref{theorem:lower_bound_wo_ind_detailed_main_part} can be seen as lower bounds for the $|L(j) \cup I(j)|$ terms. 
Even though the bounds in Theorem \ref{theorem:lower_bound_wo_ind_detailed_main_part} are ``only'' lower bounds for the ``natural’’ lower bounds in \eqref{eq_natural_lower_bound}, they are sufficient to derive, for example, the expected $\Omega(n^2)$ worst-case bound. 
Lower bounds of that type are common in the dueling bandit literature for best arm identification tasks, due to the combinatorial nature of the problem, e.g., see Theorem 5.2 in \cite{haddenhorst2021identification}.
 
Moreover, note that the lower bound in  Theorem~\ref{theorem:lower_bound_wo_ind_detailed_main_part} (i) is non-trivial (i.e., positive) for any admissible $\mathbf{P}$, and in a worst-case sense on instances $\mathbf{P}$ with $\min_{i<j}|P_{\succ}^{(i,j)}-1/2|>\Delta$ and $\text{CP}(i^\ast) = n/2 + o(n)$ of order $\Omega(\frac{n^2}{\Delta^2} \ln\frac{1}{\delta})$. Thus, existing approaches for COWI identification in a conventional dueling bandits setting are nearly optimal (cf.\ Table~7 in \cite{bengs2021preference}). The bounds in (i) and (ii) are consistent in the sense that if $\max_{j\not= i^\ast} |I(j)| = 0$, the factor $C_j$ appearing in (ii) is exactly the maximum term appearing in (i), and similarly $\Omega(\frac{n^2}{\Delta^2} \ln\frac{1}{\delta})$ samples might be necessary in expectation to identify the COWI for an indifferent $\mathbf{P}$ with $\min_{i<j}|P_{(1)}^{(i,j)} - P_{(2)}^{(i,j)}|>\Delta$. We provide the details on these deductions as well as a more sophisticated -- but more technical -- variant of (ii) in the supplementary material (Sec.\ \ref{sec:proof_lower_bound}), which is also non-trivial on some exceptional instances where \eqref{simple_indiff_lower_bound} fails to be positive.

\section{LEARNING ALGORITHM} \label{sec:learning_algorithm}

The key to design an efficient learning algorithm for the COWI identification task in general is to determine as quickly as possible which arms are potential COWIs and then restrict the sampling mechanism to the set of potential COWIs.
In light of this, an important component to ensure the efficiency of this sampling procedure is to decide quickly and reliably the allocation of the Copeland scores (cf.\ \eqref{def:Copeland_score}).
The latter can be seen in fact as finding the mode of a ternary distribution: For a fixed arm pair, say $i,j,$ the dueling feedback is governed by a ternary distribution $\tern$ with probabilities $P_{\succ}^{(i,j)}, P_{\cong}^{(i,j)}$ and $P_{\prec}^{(i,j)}$  for $i\succ j, i \cong j$ and $i \prec j.$
The Copeland score of arm $i$ is then the number of ternary distributions $\tern$ (for varying $j\neq i$) for which $P_{\succ}^{(i,j)}$ is the mode and half the number of these for which $P_{\cong}^{(i,j)}$ is the mode.
Thus, it seems reasonable to use an efficient estimation procedure for correctly identifying the mode of a discrete distribution.

\subsection{MODE IDENTIFICATION}
In \cite{jain2022pac} the \Algo{PPR-1v1} algorithm is proposed for mode identification by combining the 1-versus-1-principle from multiclass classification with prior-posterior-ratio (PPR) martingale confidence sequences \citep{waudby2020confidence}, which provide anytime confidence sequences on a specific parameter of a distribution.

\textbf{The prior-posterior-ratio martingale.}
Let $(P_\theta)_{\theta \in \Theta}$ be a family of distributions with parameter space $\Theta.$
% and $p_\theta$ denoting the respective probability mass function in case of discrete distributions, or probability density functions in case of continuous distributions.
%
Let $\pi_0$ be a prior distribution on $\Theta$ and $\pi_t$ be the posterior distribution after observing $t \in \N$ many
%
%\begin{align*}
%%	
%	\pi_t(\theta) = \frac{\pi_0(\theta) p_\theta() }{ \int_{\Theta} \pi_0(\theta')   }
%%	
%\end{align*}
%
i.i.d.\ observations $X_1,\ldots,X_t$ according to $P_{\theta^*}$ for some (unknown) $\theta^* \in \Theta.$
The prior-posterior ratio (PPR) is given by 
%
%\begin{align}
%	
	$R_t(\theta) = \nicefrac{\pi_0(\theta)}{\pi_t(\theta)}$ for  $ \theta \in \Theta.$
%	
%\end{align}
%
If $\pi_0$ assigns non-zero mass everywhere on $\Theta,$ then
\begin{align} \label{def:PPR-confidence-sequence}
	C_t = \{ \theta \, | \, R_t(\theta) < 1/\delta \} = \big\{ \theta \, \big| \, \delta <  \nicefrac{\pi_t(\theta)}{\pi_0(\theta)} \big\} 
\end{align}
is a $(1-\delta)$-confidence sequence for $\theta^*,$ i.e., it holds that $\mathbb{P}( \exists t \in \N : \, \theta^* \notin C_t ) \leq \delta$ (see \citep{waudby2020confidence}). 
The name PPR martingale stems from the fact that $(R_t(\theta^*))_{t=1}^T$ is a non-negative martingale (with respect to the canonical filtration of $X_1,X_2,\ldots,X_T$)  for any $T\in \N.$ 

\textbf{PPR-Bernoulli test.}
As an exemplary application of this result, consider the case of Bernoulli distributions for $P_\theta,$ i.e., $P_\theta = \mathrm{Ber}(\theta)$ and $\Theta=[0,1],$ for which one seeks to determine as quickly as possible (i.e., in a sequential manner) whether $\theta^*>1/2$ or $\theta^*<1/2$ holds. 
Note that this is equivalent to identifying the mode of the Bernoulli distribution $\mathrm{Ber}(\theta^*)$.
Using as the (conjugate) prior a Beta distribution with both parameters being 1, one obtains the uniform distribution on $\Theta,$ which fulfills the requirements for $C_t$ in \eqref{def:PPR-confidence-sequence} to be an anytime confidence sequence for $\theta^*.$
Further, the posterior distribution after observing $t$ many i.i.d.\ Bernoulli samples is a Beta distribution with parameters $(1+s_t(1), 1+s_t(0)),$ where $s_t(x)$ is the number of observed $x \in\{0,1\}.$
Thus, one can stop the sampling process as soon as $1/2$ is not contained in $C_t$ and declaring the $x$ with the most observations as the mode.
Formally, the PPR-Bernoulli test is to declare $x$ as the mode if $f_{\mathrm{Beta}}(1/2;s_t(x)+1,s_t(\neg x)+1)\leq \delta$ and $s_t(x)\geq s_t(\neg x),$ where $f_{\mathrm{Beta}}(\cdot;\alpha,\beta)$ is the probability density function of a Beta distribution with parameters $\alpha,\beta>0.$

\textbf{PPR-1-versus-1 test.}
At first sight, it is tempting to instantiate the prior-posterior-ratio (PPR) martingale approach with the Dirichlet distribution as the conjugate prior for a categorical distribution to identify the latter's mode by stopping the sampling process similarly as for the PPR-Bernoulli test above.
However, if the categorical distribution has more than two categories, say $c_1,c_2,\ldots,c_K$ with $K \in \N_{\geq 2},$ then it is difficult to obtain a closed-form criterion as in the PPR-Bernoulli test, so that one needs to resort to a numerical computation.

In light of this, in \cite{jain2022pac} it is proposed to reduce the mode identification problem to multiple PPR-Bernoulli tests using the 1-versus-1-principle from multiclass classification.
%\begin{minipage}{0.44\textwidth}
%
%
Thus, a PPR-Bernoulli test is simultaneously conducted for each pair of categories $(c_i,c_j)$ with $i\neq j$ and an error probability of $\delta/(K-1),$ each of which uses only the number of occurrences of $c_i$ and $c_j$ and ignoring the remaining ones.
If there exists a category that has won all of its tests, it is declared to be the mode.
This procedure is equivalent to monitoring only the PPR-Bernoulli test for the pair of categories $(c_{t(1)},c_{t(2)}),$ where  $c_{t(1)}$ resp.\ $c_{t(2)}$ is the category that has the  most resp.\ second most occurrences after observing $t$ samples.
%This procedure is equivalent to monitor only the PPR-Bernoulli test for the pair of categories $(c_{t(1)},c_{t(2)}),$ where  $c_{t(1)}$  is the category that has the most occurrences, while  $c_{t(2)}$ is the one with the second-most occurrences after observing $t$ samples.
%
Consequently, the prior-posterior-ratio-1-versus-1 (\Algo{PPR-1v1}) test decision is to declare $c_{t(1)}$ as the mode if $f_{\mathrm{Beta}}(1/2;s_{t(1)}+1,s_{t(2)}+1)\leq \delta/(K-1),$ where $s_{t(x)}$ denotes the number of occurrences of $c_{t(x)}.$
%
%\end{minipage}
%
%\hfill
%
%\begin{minipage}{0.52\textwidth}
%\begin{algorithm}[H]
%	%	[H]
%	\caption{\Algo{PPR-1v1}} \label{alg:PPR}
%	%	
%	\begin{algorithmic}[1]
%		%		
%		\STATE \textbf{Input:} Arms $i$ and $j,$ error prob.\ $\delta\in(0,1)$
%		%		
%		\STATE \textbf{Initialization:}  $S=(S_1,S_2,S_3) \leftarrow (0,0,0)$
%		%		
%		\WHILE{ \textbf{TRUE} }
%		%					
%		\STATE Compare $i$ and $j$ 
%		%			
%		  \STATE Observe $o\in \{i\succ j, i \cong j, i\prec j\}$ 
%		%			
%		\IF{$o= i\succ j$} 
%		%			
%		\STATE $S_1 \leftarrow S_1 +1$
%		%
%		\ELSIF{$o= i \cong j$}
%		%			
%		\STATE $S_2 \leftarrow S_2 +1$
%		%		
%		\ELSE
%		%		
%		\STATE $S_3 \leftarrow S_3 +1$
%		%		
%		\ENDIF
%		%			
%		\IF{$f_{\mathrm{Beta}}(1/2;S_{(1)}+1,S_{(2)}+1)\leq \delta/2  $}
%		%			
%		\STATE \textbf{return} $\argmax_{k=1,2,3} S_k$
%		%			
%		\ENDIF
%		\ENDWHILE
%	\end{algorithmic}
%\end{algorithm}
%\end{minipage}

The probability of making an error with this test procedure, i.e., not declaring the true mode as the mode of the underlying categorical distribution, is bounded by means of a union bound by $\delta.$
Moreover, the above \Algo{PPR-1v1} test is asymptotically optimal in the sense that the ratio of its expected stopping time and the lower bound on the expected stopping time for a fixed categorical distribution tends to one for the error probability $\delta$ tending to zero \cite{jain2022pac}.

We can transfer this test procedure to the case of identifying the mode of a ternary distribution $\tern,$ which is equivalent to finding the mode of a categorical distribution with three categories $c_1:= ``i\succ j'', c_2:= ``i \cong j''$ and $c_3:= ``i \prec j''.$
The explicit \Algo{PPR-1v1} procedure for this case is given in Algorithm \ref{alg:PPR}, where $S_{(1)} \geq S_{(2)} \geq S_{(3)}$ is the order statistics of $S_1,S_2$ and $S_3.$

\begin{algorithm}[H]
	%	[H]
	\caption{\Algo{PPR-1v1}} \label{alg:PPR}
	\begin{algorithmic}[1]
		\STATE \textbf{Input:} Arms $i$ and $j,$ error prob.\ $\delta\in(0,1)$
		\STATE \textbf{Initialization:}  $S=(S_1,S_2,S_3) \leftarrow (0,0,0)$
		\WHILE{ \textbf{TRUE} }
		\STATE Compare $i$ and $j$ 
		\STATE Observe $o\in \{i\succ j, i \cong j, i\prec j\}$ 
		\IF{$o= i\succ j$} 
		\STATE $S_1 \leftarrow S_1 +1$
		\ELSIF{$o= i \cong j$}
		\STATE $S_2 \leftarrow S_2 +1$
		\ELSE
		\STATE $S_3 \leftarrow S_3 +1$
		\ENDIF
		\IF{$f_{\mathrm{Beta}}(1/2;S_{(1)}+1,S_{(2)}+1)\leq \delta/2  $}
		\STATE \textbf{return} $\argmax_{k=1,2,3} S_k$
		\ENDIF
		\ENDWHILE
	\end{algorithmic}
\end{algorithm}

%PPR martingale confidence sequences \citep{jain2022pac}

\subsection{POCOWISTA}
Guided by the design idea above, i.e., determining as quickly as possible which arms are potential COWIs and then restricting the sampling mechanism to the set of potential COWIs, we propose the \Algo{POCOWISTA} (POtential COpeland WInner STays Algorithm) in Algorithm \ref{alg:pocowista}.
For each arm two Copeland scores are maintained: (i) the current Copeland score $\widehat{CP}(\cdot)$, which is determined (using the PPR-1v1 algorithm) by the duels already contested, and (ii) the potential Copeland score $\overline{CP}(\cdot)$, which is determined by both the duels already contested and the duels not yet contested, which add an optimistic bonus to the current Copeland score. 
For the calculation of this optimistic bonus, a set of arms $D(\cdot)$ is maintained for each arm,  which includes the arms that have already been compared to (dueled with) the respective arm as well as the arm itself. 
The optimistic bonus for an arm, say $i$, is then the size of the set of arms not compared to $i$ so far, i.e., $|\mathcal{A} \setminus D(i)| = n - |D(i)|.$
\begin{minipage}{0.52\textwidth}
\begin{algorithm}[H] 
	%	[H]
	\caption{\Algo{POCOWISTA}} \label{alg:pocowista}
	\begin{algorithmic}[1]
		\STATE \textbf{Input:} Set of arms $\mathcal{A},$ error prob.\ $\delta\in(0,1)$
		\STATE \textbf{Initialization:} $t\leftarrow 1$ and for each $i\in\mathcal{A}$ set
		
		$D(i) \leftarrow \{i\}$ \hfill (set of already compared arms)
		
		$\widehat{CP}(i) \leftarrow 0 $ \hfill(current Copeland score)
		
		$\overline{CP}(i) \leftarrow n - 1$ \hfill(potential Copeland score)
		\WHILE{$\nexists i$ s.t.\ $\widehat{CP}(i) \geq \overline{CP}(j) \, \forall j \in \mathcal{A}\setminus\{i\} $ }
		\STATE $i_t = \argmax_{i\in \mathcal{A}}  \overline{CP}(i) $
		\STATE $j_t = \argmax_{j\in \mathcal{A}\setminus D(i_t) }  \widehat{CP}(j) $
		\STATE $k \leftarrow \Algo{PPR-1v1}(i_t,j_t,\delta/  \binom{n}{2} )$
		\STATE $\Algo{Scores-Update}(i_t,j_t,k)$
		\STATE $t \leftarrow t +1$
		\ENDWHILE
		\STATE \textbf{return} $\argmax_{i \in \mathcal{A}} \widehat{CP}(i) $
	\end{algorithmic}
\end{algorithm}
\end{minipage}
\hfill
\begin{minipage}{0.47\textwidth} 
\begin{algorithm}[H]
	\caption{\Algo{Scores-Update}} \label{alg:Score-Update}
	\begin{algorithmic}[1]
		\STATE \textbf{Input:} Arms $i,$ $j,$ ternary decision $k\in\{1,2,3\}$
		\IF{$k=1$} 
		\STATE $\widehat{CP}(i) \leftarrow \widehat{CP}(i) +1$
		\ELSIF{$k=2$}
		\STATE $\widehat{CP}(i) \leftarrow \widehat{CP}(i) + \frac12$ \\
		$\widehat{CP}(j) \leftarrow \widehat{CP}(j) + \frac12$
		\ELSE
		\STATE $\widehat{CP}(j) \leftarrow \widehat{CP}(j) +1$
		\ENDIF
		\STATE $D(i) \leftarrow D(i) \cup \{j\},$ $D(j) \leftarrow D(j) \cup \{i\}$
		\STATE $\overline{CP}(i) \leftarrow n - |D(i)| + \widehat{CP}(i)$
		\STATE $\overline{CP}(j) \leftarrow n - |D(j)| + \widehat{CP}(j)$ \\
	\end{algorithmic}
\end{algorithm}
\end{minipage}

The algorithm proceeds in rounds, in each of which it uses as the ``first'' arm the current incumbent in terms of the potential Copeland scores (line 4, Algo.\ \ref{alg:pocowista}) and as the ``second'' arm the one with the highest current Copeland score among those not yet compared to the first (line 5, Algo.\ \ref{alg:pocowista}).
These two arms are successively dueled against each other until the mode of their feedback distribution is identified by means of the PPR-1v1 algorithm (line 6, Algo.\ \ref{alg:pocowista}), which leads to an update of their Copeland scores (line 7, Algo.\ \ref{alg:pocowista}): The current Copeland score of the dominating arm is increased by one (lines 2--3,6--7, Algo.\ \ref{alg:Score-Update}), while in case of an indifference both obtain half a point (lines 4--5, Algo.\ \ref{alg:Score-Update}).
In addition, the set of already compared arms of the two arms is extended by the other one (line 9, Algo.\ \ref{alg:Score-Update}) and the potential Copeland scores are updated as well (lines 10--11, Algo.\ \ref{alg:Score-Update}).

The update procedure corresponds to the end of a round, which leads to the start of a new round, unless there is an arm whose current Copeland score is not smaller than any other potential Copeland score (line 3, Algo.\ \ref{alg:pocowista}). 
In such a case the arm is a COWI and returned by the algorithm (line 10).
%(line 10, Algo.\ \ref{alg:pocowista}).

For the sampling complexity of \Algo{POCOWISTA} we derive the following result (proof in Sec. \ref{sec:proof_theorem_pocowista}).
%
%which shows that it almost matches the lower bound in Theorem \ref{theorem:lower_bound_wo_ind_detailed_main_part}.
%

\begin{theorem}  \label{theorem_pocowista}
	%	
	%	Assume that there exists no pair $i,j \in \mathcal{A}$ with $i\neq j$ such that $P_{\succ}^{(i,j)} = 1/3$ and $P_{\cong}^{(i,j)} = 1/3.$ 
	%	
	Let $\Algo{A}:= \Algo{POCOWISTA}.$
	For any dueling bandits problem with indifferences characterized by $\mathbf{P} = ((P_{\succ}^{(i,j)},P_{\cong}^{(i,j)}, P_{\prec}^{(i,j)}))_{i<j},$    such that there exists no pair $i,j \in \mathcal{A}$ with $i\neq j$ and $P_{\succ}^{(i,j)} = P_{\succ}^{(j,i)}= 1/3,$  it holds that
	%	
	%\begin{align*}
		%		
		$$\prob\big(\hat{i}_{\Algo{A}} \in \mathcal{C}(\mathbf{P}) \mbox{ and } \tau^{\Algo{A}}(\mathbf{P}) \leq t(\mathbf{P},\delta) \big) 
		\geq 1- \delta,		$$
		%		
	%\end{align*}
	%	
    %    where $t(\mathbf{P},\delta) 
	%	= \sum\nolimits_{ t=1 }^T t_0\big(  (P_{\succ}^{(i_t,j_t)}, P_{\cong}^{(i_t,j_t)},P_{\prec}^{(i_t,j_t)}), \nicefrac{\delta}{\binom{n}{2}} \big),$ 
	where $t(\mathbf{P},\delta) 
		\leq \sum\nolimits_{i<j} t_0\big(  (P_{\succ}^{(i,j)}, P_{\cong}^{(i,j)},P_{\prec}^{(i,j)}), \nicefrac{\delta}{\binom{n}{2}} \big),$ 
	\begin{align} \label{sample_complexity_PPR1v1}
		t_0\big(  (p_1,p_2,p_3), \delta \big) 
            = c_1 p_{(1)}	\ln\Big(	\sqrt{\frac{2\, c_2}{\delta}   } \cdot  \frac{p_{(1)}	}{(p_{(1)}- p_{(2)})} 		\Big)(p_{(1)}- p_{(2)})^{-2}, 
		%= \frac{c_1 p_{(1)}	\ln\left(	\sqrt{\frac{2\, c_2}{\delta}   } \frac{p_{(1)}	}{(p_{(1)}- p_{(2)})} 		\right)}{(p_{(1)}- p_{(2)})^2}, 
		%		
	\end{align}
	$p_{(1)}\geq p_{(2)} \geq p_{(3)}$ is the order statistic of $p_1,p_2,p_3 \in [0,1],$ $c_1 = 194.07,$ and  $c_2 = 79.86.$ 
\end{theorem}
Here, $t_0\big(  (p_1,p_2,p_3), \delta \big)$ is the sample complexity of  \Algo{PPR-1v1} to identify the mode of a (categorical) distribution with probabilities $p_1,p_2,p_3$ with an error probability of at most $\delta$ (see Theorem 9 in \cite{jain2022pac}).
Since in the worst case, one needs to ensure for all pairs of arms $i,j$ that the mode of the ternary distribution $\tern$ is correctly identified, the \Algo{PPR-1v1} algorithm is used with $\nicefrac{\delta}{\binom{n}{2}}$ for its error probability (line 6 of Algo.\ \ref{alg:pocowista}) to ensure that the overall error probability $\delta$ is not exceeded.
If $\mathbf{P} = ((P_{\succ}^{(i,j)},P_{\cong}^{(i,j)}, P_{\prec}^{(i,j)}))_{i<j}$ is such that $\min_{i<j}|P_{(1)}^{(i,j)} - P_{(2)}^{(i,j)}|>\Delta,$ we see that \Algo{POCOWISTA}'s sample complexity is in  $O(\frac{n^2 \ln(n)}{\Delta^2} \ln\frac{1}{\delta})$ which almost matches the worst-case lower bound in Theorem \ref{theorem:lower_bound_wo_ind_detailed_main_part}.
Thus, for the conventional dueling bandit case it has the same worst-case sample complexity as existing algorithms such as SAVAGE \citep{UrClFeNa13} or PBR-CCSO \citep{BuSzWeChHu13}.
%
% by \Algo{POCOWISTA}.
%
%Lastly, the condition on $\mathbf{P}$  in Theorem \ref{theorem_pocowista} ensures that each ternary distribution $\tern$ has a unique mode.
%%
%If this condition is not fulfilled, the lower bounds in Theorem \ref{theorem:lower_bound_wo_ind_detailed_main_part} imply that any learning algorithm has an infinite sample complexity.
A comparison of the sample complexity per pair, say i and j, of all the three algorithms with respect to the number of arms $n,$ error probability $\delta$ and the gap between the arms $ \Delta_{i,j} = |P_{\succ} (i,j) - 1/2|$ for the conventional dueling bandits (without indifferences) is given in the following table:

\begin{table} [ht!]
%	\label{tab_per_pair_comparison}
	\centering
	\begin{tabular}{c|c|c}
%		\hline
		\Algo{POCOWISTA} & \Algo{SAVAGE} & \Algo{PBR-CCSO} \\
		\hline
		$ \frac{1}{\Delta_{i,j}^2} \ln\Big( \frac{n }{ \sqrt{\delta} } \cdot \frac{1}{ \Delta_{i,j}} \Big)$ & $ \frac{1}{\Delta_{i,j}^2} \ln\Big( \frac{n^2}{\delta } \cdot \frac{1}{ \Delta_{i,j}} \Big)$ & $ \frac{1}{\Delta_{i,j}^2} \ln\Big( \frac{n}{\delta } \cdot \frac{1}{ \Delta_{i,j}} \Big) $ \\
		
	\end{tabular}
\end{table}

There, we see that \Algo{POCOWISTA} and SAVAGE have a better dependence on $n$ compared to PBR-CCSO, while \Algo{POCOWISTA} additionally has a better dependence on $\delta$ than the other two. Accordingly, \Algo{POCOWISTA} is expected to have a better sample complexity in practical applications than SAVAGE, which in turn has a better sample complexity than PBR-CCSO. This is supported by our experimental results in Section \ref{sec:experiments}.

%Comparing \Algo{POCOWISTA}’s sample complexity per pair, say i and j, with respect to the number of arms $n,$ error probability $\delta$ and the gap between the arms $ \Delta_{i,j} = |P_{\succ} (i,j) − 1/2|$ for the conventional dueling bandits (without indifferences) we obtain that it is proportional to $$ \frac{1}{\Delta_{i,j}^2} \ln\Big( \frac{n }{ \sqrt{\delta} } \cdot \frac{1}{ \Delta_{i,j}} \Big).$$ The per pair sample complexity with respect to $n,\delta$ and $\Delta_{i,j}$ of PBR-CCSO [29] is proportional to $$ \frac{1}{\Delta_{i,j}^2} \ln\Big( \frac{n^2}{\delta } \cdot \frac{1}{ \Delta_{i,j}} \Big),$$ while SAVAGE’s [30] per pair sample complexity is proportional to $$ \frac{1}{\Delta_{i,j}^2} \ln\Big( \frac{n}{\delta } \cdot \frac{1}{ \Delta_{i,j}} \Big). $$

\section{TRANSITIVITY OF PREFERENCES} \label{sec:transitivity}

Although the sample complexity of \Algo{POCOWISTA} almost matches the lower bound, it is in some sense unsatisfactory as it is log-linear with respect to the number of actions (the number of arm pairs) in the worst case, i.e., $O(n^2 \ln(n))$.
In light of this, we consider in this section structural properties on the preference probabilities such that the dependency of \Algo{POCOWISTA}'s sample complexity with respect to $n$ is reduced to $O(n \ln(n))$ in the worst case. 
To this end, we leverage the commonly used stochastic transitivity properties in conventional dueling bandits, which in essence assume that if an arm $i$ is preferred over arm $j$ and arm $j$ over arm $k,$ then $i$ is also preferred over arm $k.$
As this notion of transitivity merely considers the part of the feedback regarding the strict preferences, we augment it with the notion of IP-transitivity and PI-transitivity as well as transitivity of indifferences \cite{sep-preferences} in order to account for the possibility of observing indifferences.
\begin{definition} \label{def:transitivity}
	A dueling bandits problem with indifferences characterized by $\mathbf{P} = ((P_{\succ}^{(i,j)},P_{\cong}^{(i,j)}, P_{\prec}^{(i,j)}))_{i<j}$  is called \emph{transitive} if the underlying preference probabilities on the set of arms $\mathcal{A}$ satisfy for any distinct $i,j,k \in \mathcal{A}:$
	\begin{enumerate}
		[noitemsep,topsep=0pt,leftmargin=8mm]
		\item (Transitivity of strict preference) If	$ P_{\succ}^{(i,j)} > \max\big(	P_{\prec}^{(i,j)}, P_{\cong}^{(i,j)}		\big)$ and  $
		P_{\succ}^{(j,k)} >  \max\big(	P_{\prec}^{(j,k)}, P_{\cong}^{(j,k)}		\big),$ then \\ $ P_{\succ}^{(i,k)} > \max\big(	P_{\prec}^{(i,k)}, P_{\cong}^{(i,k)}		\big).$
		\item (IP-transitivity) If	$ P_{\cong}^{(i,j)} > \max\big(	P_{\prec}^{(i,j)}, P_{\succ}^{(i,j)}		\big)$ and  $
		P_{\succ}^{(j,k)} >  \max\big(	P_{\prec}^{(j,k)}, P_{\cong}^{(j,k)}		\big),$ then \\ $ P_{\succ}^{(i,k)} > \max\big(	P_{\prec}^{(i,k)}, P_{\cong}^{(i,k)}		\big).$
		\item (PI-transitivity) If	$ P_{\succ}^{(i,j)} > \max\big(	P_{\prec}^{(i,j)}, P_{\cong}^{(i,j)}		\big)$ and  $
		P_{\cong}^{(j,k)} >  \max\big(	P_{\prec}^{(j,k)}, P_{\succ}^{(j,k)}		\big),$ then \\ $ P_{\succ}^{(i,k)} > \max\big(	P_{\prec}^{(i,k)}, P_{\cong}^{(i,k)}		\big).$
		\item (Transitivity of indifference)  If	$ P_{\cong}^{(i,j)} > \max\big(	P_{\prec}^{(i,j)}, P_{\succ}^{(i,j)}		\big)$ and  $
		P_{\cong}^{(j,k)} >  \max\big(	P_{\prec}^{(j,k)}, P_{\succ}^{(j,k)}		\big),$ then \\ $ P_{\cong}^{(i,k)} > \max\big(	P_{\prec}^{(i,k)}, P_{\succ}^{(i,k)}		\big).$
	\end{enumerate}
\end{definition}
For the conventional dueling bandit setting, i.e., where $P_{\cong} = \mathbf{0} \in [0,1]^{n\times n},$ the transitivity property in Definition \ref{def:transitivity} is equivalent to strict weak stochastic transitivity.

Under the assumption of a transitive dueling bandits problem with indifferences as stipulated by Definition \ref{def:transitivity}, one can modify the update rule of \Algo{POCOWISTA} for the current and potential Copeland scores to take the transitivity property into account (see Algo.\ \ref{alg:Trans-Score-Update}).
More specifically, the transitivity of strict preferences and the IP-transitivity imply that once it is statistically ensured that an arm $i$ either dominates another one $j$ or is indifferent to it, then $i$ will also dominate all arms dominated by $j$.
PI-transitivity implies further that in case of dominance of $i$ over $j$, all arms indifferent to $j$ are dominated by $i$ as well.
Thus, if $i$ dominates $j$ its current Copeland score can be increased, in addition to the update due to the domination of $j$ (increasing by one), by the number of arms that are dominated/indifferent by/to $j$ (line 3).
In case of indifference between $i$ and $j$ the current Copeland scores can be increased, in addition to the update due to the indifference (increasing both by one half), by the number of arms that are dominated/indifferent to the other (lines 9--10).
Conversely, the potential Copeland score of the dominated arm $j$ can be decreased, in addition to the update due to the domination by $i$, by the number of arms that are dominating/indifferent to $i.$
In order to perform the updates efficiently three additional sets are maintained, namely the set of defeated arms $W(\cdot),$ the set of indifferent arms $I(\cdot)$, and the set of inferior arms $L(\cdot)$ (line 2 of Algo.\ref{alg:trapocowista}).
%The transitivity of indifference allows to update 

%
\begin{minipage}{0.48\textwidth}
\begin{algorithm}[H] 
	\caption{\Algo{TRA-POCOWISTA}} \label{alg:trapocowista}
	\begin{algorithmic}[1]
		\STATE \textbf{Input:} Set of arms $\mathcal{A},$ error prob.\ $\delta\in(0,1)$
		\STATE \textbf{Initialization:}  $t\leftarrow 1$ and for each $i\in\mathcal{A}$ 		
		$D(i) \leftarrow \{i\}$ \hfill %(set of already compared arms)
		
		$W(i) \leftarrow \emptyset$ \hfill (set of defeated arms)
		
		$I(i) \leftarrow \emptyset$ \hfill(set of indifferent arms)
		
		$L(i) \leftarrow \emptyset$ \hfill(set of superior arms)
		
		$\widehat{CP}(i) \leftarrow 0 $ \hfill(current Copeland score)
		
		$\overline{CP}(i) \leftarrow n - 1$ \hfill(potential Copeland score)

		\WHILE{$\nexists i$ s.t.\ $\widehat{CP}(i) \geq \overline{CP}(j) \, \forall j \in \mathcal{A}\setminus\{i\} $ }
		\STATE $i_t = \argmax_{i\in \mathcal{A}}  \overline{CP}(i) $
		\STATE $j_t = \argmax_{j\in \mathcal{A}\setminus D(i_t) }  \widehat{CP}(j) $
		\STATE $k \leftarrow \Algo{PPR-1v1}(i_t,j_t,\delta/n)$
		\STATE $\Algo{Transitive-Score-Update}(i_t,j_t,k)$
		\STATE $t \leftarrow t +1 $
		\ENDWHILE
		\STATE \textbf{return} $\argmax_{i \in \mathcal{A}} \widehat{CP}(i) $
	\end{algorithmic}
\end{algorithm}
\end{minipage}
\hfill
\begin{minipage}{0.51\textwidth}
\begin{algorithm}[H]
	\caption{\Algo{Transitive-Score-Update}} \label{alg:Trans-Score-Update}
	\begin{algorithmic}[1]
		\STATE \textbf{Input:} Arms $i,$  $j,$  $k\in\{1,2,3\}$
		\IF{$k=1$} 
		%			
		%		\STATE $\widehat{CP}(i) \leftarrow \widehat{CP}(i) + | W(j) \cup I(j) \setminus D(i) | +1$
		\STATE $\widehat{CP}(i) \leftarrow \widehat{CP}(i) + | W(j) \cup I(j)  | +1$
		%		\STATE $\widehat{CP}(i) \leftarrow \widehat{CP}(i) + \widehat{CP}(j) +1$
		%			
		\STATE $W(i) \leftarrow W(i) \cup W(j) \cup I(j)  \cup \{j\}$
		\STATE $D(i) \leftarrow D(i) \cup W(j) \cup I(j)  \cup \{j\}$
		\STATE $L(j) \leftarrow L(j) \cup L(i) \cup I(i)  \cup \{i\}$
		\STATE $D(j) \leftarrow D(j) \cup L(i) \cup I(i)  \cup \{i\}$
		\ELSIF{$k=2$}
		%			
		%		\STATE $\widehat{CP}(i) \leftarrow \widehat{CP}(i) + | W(j)  \setminus D(i) | +1/2 (1+ |I(j) \setminus D(i) |)$
		%		\STATE $\widehat{CP}(i) \leftarrow \widehat{CP}(i) + | W(j)    | +1/2 (1+ |I(j)  |)$
		\STATE $\widehat{CP}(i) \leftarrow \widehat{CP}(i) + | W(j)    | +1/2 (1+ |I(j)  |)$
		%			
		%\STATE $\widehat{CP}(j) \leftarrow \widehat{CP}(j) + | W(i)  \setminus D(j) | +1/2 (1+ |I(i) \setminus D(j) |)$
		\STATE $\widehat{CP}(j) \leftarrow \widehat{CP}(j) + | W(i)   | +1/2 (1+ |I(i)   |)$
		\STATE $W(i) \leftarrow W(i) \cup W(j), W(j) \leftarrow W(i)  $
		\STATE $L(i) \leftarrow L(i) \cup L(j), L(j) \leftarrow L(i)   $
		\STATE $I(i) \leftarrow I(i) \cup I(j) \cup \{i\} , I(j) \leftarrow I(i)  $
		\STATE $D(i) \leftarrow D(i) \cup D(j) \cup \{j\},$ $D(j) \leftarrow D(i)$ 
        %\STATE $\widehat{CP}(i) \leftarrow | W(i)    | +1/2 (1+ |I(i)  |),$ $\widehat{CP}(j) \leftarrow \widehat{CP}(j)$
		%		\STATE $D(j) \leftarrow D(i) \cup D(j) \cup \{i\}$
		%		
		\ELSE
		\STATE Same as for $k=1$ with $i$ and $j$ reversed
		%			$\widehat{CP}(j) \leftarrow \widehat{CP}(j) + | W(i) \setminus D(j) | +1$
		%			\STATE $W(j) \leftarrow W(i) \cup W(j) \cup \{i\}$
		%		
		\ENDIF
		%	
		%			\STATE $D(i) \leftarrow D(i) \cup \{j\},$ $D(j) \leftarrow D(j) \cup \{i\}$
		%		\STATE $\overline{CP}(i) \leftarrow n - |D(i)| + \widehat{CP}(i)$
		%		\STATE $\overline{CP}(j) \leftarrow n - |D(j)| + \widehat{CP}(j)$
		\STATE Same steps as line 10 and 11 in Algorithm \ref{alg:Score-Update}
      %  \STATE $\overline{CP}(i) \leftarrow n - |D(i)| + \widehat{CP}(i)$ \\
        %
        %\STATE $\overline{CP}(j) \leftarrow n - |D(j)| + \widehat{CP}(j)$
		%			
	\end{algorithmic}
\end{algorithm}
\end{minipage}

We obtain the following improved bound on the sample complexity w.r.t. $n$ (proof in Sec.\ \ref{sec:proof_theorem_tra_pocowista}).
%We obtain the following improved high probability bound on the sample complexity  (proof in Sec.\ \ref{sec:proof_theorem_tra_pocowista}).

\begin{theorem} \label{theorem_tra_pocowista}
%	
	%	
%	Assume that there exists no pair $i,j \in \mathcal{A}$ with $i\neq j$ such that $P_{\succ}^{(i,j)} = 1/3$ and $P_{\cong}^{(i,j)} = 1/3.$ 
%	
	Let $\Algo{A}:= \Algo{TRA-POCOWISTA}.$
	For any dueling bandits problem with indifferences as in Theorem \ref{theorem_pocowista} which in addition is transitive according to Definition \ref{def:transitivity},  it holds that
	\begin{align*}
		\prob\big(\hat{i}_{\Algo{A}} \in \mathcal{C}(\mathbf{P}) \mbox{ and } \tau^{\Algo{A}}(\mathbf{P}) \leq \tilde t(\mathbf{P},\delta) \big) 
		\geq 1- \delta,		
	\end{align*}
	where 
	%	
	%\begin{center}
		%		
	$	\tilde t(\mathbf{P},\delta) 
		= \sum\nolimits_{t=1}^T t_0\big(  (P_{\succ}^{(i_t,j_t)}, P_{\cong}^{(i_t,j_t)},P_{\prec}^{(i_t,j_t)}), \nicefrac{\delta}{n} \big),$
		%		  
		%		
	%\end{center}
	%
	$t_0$ is as in \eqref{sample_complexity_PPR1v1} and $T \leq n.$
\end{theorem}

\section{EXPERIMENTS} \label{sec:experiments}

Since, to the best of our knowledge, there are no algorithms for identification tasks in dueling bandits problems with indifferences, we resort in the following experiments to the conventional dueling bandits problem, i.e., where $P_{\cong} = \mathbf{0}.$ 
First, we compare \Algo{POCOWISTA} with \Algo{SAVAGE} \citep{UrClFeNa13} and \Algo{PBR-CCSO} \citep{BuSzWeChHu13}, which are the only available methods for the task of identifying a Copeland winner. 
Since the Copeland set boils down to a singleton set consisting of the Condorcet winner in case of the latter's existence, we compare \Algo{POCOWISTA} also with the state-of-the-art algorithm \Algo{SELECT} \citep{MoSuEl17} and \Algo{DKWT} \citep{haddenhorst2021identification} for identifying a Condorcet winner.
% by restricting to problem instances with an existing Condorcet winner.

\textbf{Copeland Winner Identification.}
Given a strict preference probability matrix $P_{\succ} \in \mathcal{P}_{\succ}(n)$ for $n$ arms, where 
$\mathcal{P}_{\succ}(n) := \big\{   P_{\succ} \in [0,1]^{n\times n} \, \big| \, P_{\succ}^{(i,j)}+P_{\succ}^{(j,i)}=1, \forall i\neq j  \big\},$ we consider first the setting of identifying a COWI of $P_{\succ}.$
%
%We consider first the setting of identifying a COWI from a given strict preference probability matrix $P_{\succ} \in \mathcal{P}_{\succ}(n)$ for $n$ arms, where 
%%
%$\mathcal{P}_{\succ}(n) := \big\{   P_{\succ} \in [0,1]^{n\times n} \, \big| \, P_{\succ}^{(i,j)}+P_{\succ}^{(j,i)}=1, \forall i\neq j  \big\}.$
%%
We distinguish between two classes of a (conventional) dueling bandits problem with different difficulties for this task: 
%$\mathcal{P}_1(n) :=  \big\{ P_{\succ} \in \mathcal{P}_{\succ}(n) \, \big|  |P_{\succ}^{(i,j)} - 1/2|\geq 0.1 \,    \forall i\neq j  \big\}$ and $\mathcal{P}_2(n) :=  \big\{ P_{\succ} \in \mathcal{P}_{\succ}(n) \, \big|  |P_{\succ}^{(i,j)} - 1/2|\geq 0.05,  | P_{\succ}^{(i,j)} - 1| \leq 0.3, \,  \forall i\neq j  \big\}.$
%\begin{itemize}
%	[noitemsep,topsep=0pt,leftmargin=4mm]
%	
\begin{align*}
	&\mathcal{P}_1(n) :=  \big\{ P_{\succ} \in \mathcal{P}_{\succ}(n) \, \big|  |P_{\succ}^{(i,j)} - 1/2|\geq 0.1 \,    \forall i\neq j  \big\}, \\
	&\mathcal{P}_2(n) :=  \big\{ P_{\succ} \in \mathcal{P}_{\succ}(n) \, \big|  |P_{\succ}^{(i,j)} - 1/2|\geq 0.05,  | P_{\succ}^{(i,j)} - 1| \leq 0.3, \,  \forall i\neq j  \big\}. 
\end{align*}
The class $\mathcal{P}_1(n)$ are easy problems, as the gap parameters are quite large, while $\mathcal{P}_2(n)$ consists of more difficult problem instances, where the pairwise probabilities are close to $1/2$ making it more difficult to identify whether one arm dominates the other or vice versa.

For both classes, a strict preference probability matrix is repeatedly selected uniformly at random, and then used to generate the feedback for the learning algorithms for a total of 100 repetitions.
%For both classes, a strict preference probability matrix is repeatedly selected uniformly at random, and then passed to \Algo{POCOWISTA}, \Algo{TRA-POCOWISTA}, \Algo{SAVAGE}  and \Algo{PBR-CCSO} for a total of 100 repetitions.
%
Note that all algorithms are parameter-free in the sense that they only need the desired error probability $\delta$ and the number of arms $n,$  but no other adjustable (hyper-)parameters to be executed.

The first four columns of Table \ref{tab:COPE_IDENT_TAKS} show the average sample complexity (and their standard deviation in brackets) of the four algorithms for the two problem classes with $n=20$ arms and different choices for the error probabilities $\delta.$
All algorithms have an empirical error probability of zero for each $\delta.$
This is due to the Bonferroni correction used by each algorithm to ensure that each pairwise comparison is correctly decided, making the overall decision quite conservative.
As can be seen from the table, the average sample complexities of our algorithms are by a magnitude smaller than for the existing methods and the same holds for their standard deviations.
\Algo{TRA-POCOWISTA} has even a clear improvement over \Algo{POCOWISTA}, although the considered problem instances do not necessarily satisfy the transitivity property in Definition \ref{def:transitivity}.
%
%Finally, the results are similar in terms of relative comparison for different sizes of $n$, so we only explicitly report the results for $n=20$. 

\begin{table}[ht!] \label{tab:COPE_IDENT_TAKS}
	\caption{Average sample complexities (and standard deviations) for the considered dueling bandit classes of the considered learning algorithms.}
	\centering
	\resizebox{0.999\textwidth}{!}{
		\begin{tabular}{rrrr||rrrr}
			%				\hline
			$\mathcal{P}_1(20)$                  & $\delta = 0.01$ & $\delta =0.05$ & $\delta = 0.1$ & 
                $\mathcal{P}_1^{\mathrm{CW}}(20)$    & $\delta = 0.01$ & $\delta =0.05$ & $\delta = 0.1$  \\
			\hline
			\Algo{PO} & 4208.35 (1685.052) & 3620.81 (1394.413) & 3367.60 (1427.909) & 
                \Algo{PO} & 4180.12 (1547.221) & 3580.90 (1377.8151) & 3387.84 (1267.044)  \\  
			\Algo{TRA-PO} & 2179.94 (638.675) & 1853.66 (543.604) & 1665.52 (499.886) &
                \Algo{TRA-PO}  & 2107.97 (608.020) & 1794.57 (547.680) & 1703.58 (499.242) \\  
			\Algo{SAVAGE} & 29197.83 (4039.559) & 26640.51 (3467.904) & 25430.08 (3341.993) &
                \Algo{SELECT} & 3710.81 (135.899) & 2720.90 (98.259) & 2302.20 (85.289)  \\
			\Algo{PBR-CCSO} & 37144.10 (4239.656) & 34448.98 (4019.530 )& 33205.52 (3821.270) & 
			\Algo{DKWT} & 30210.76 (5071.470) & 24815.17 (5120.585) & 23096.64 (4413.633)  \\ 
			\hline
			\\
			$\mathcal{P}_2(20)$                  & $\delta = 0.01$ & $\delta =0.05$ & $\delta = 0.1$ &
			$\mathcal{P}_2^{\mathrm{CW}}(20)$    & $\delta = 0.01$ & $\delta =0.05$ & $\delta = 0.1$ \\  
			\hline
			\Algo{PO} & 81325.20 (21299.856) & 69218.40 (17068.274) & 66235.00 (15301.290) &  
			\Algo{PO} & 23759.41 (8528.608) & 20598.68 (7673.594) & 19478.87 (7077.146)  \\ 
			\Algo{TRA-PO} & 10390.70 (2232.592) & 9346.40 (2673.487) & 9223.70 (2601.344) &
			\Algo{TRA-PO} & 11691.91 (2325.349) & 10106.83 (2212.905) & 9260.47 (1955.686)  \\  
			\Algo{SAVAGE} & 199606.30 (23280.419) & 189721.10 (31101.592) & 183988.60 (22971.049) &
			\Algo{SELECT} & 18613.13 (325.965) & 13678.96 (238.923) & 11558.86 (190.281)  \\  
			\Algo{PBR-CCSO} & 248610.00 (24681.748) & 228844.60 (21642.366) & 222940.10 (22377.070) &
			\Algo{DKWT} & 115587.58 (33281.180) & 94174.69 (27224.825) & 92921.00 (29351.814) \\  
			\hline
	\end{tabular}}
\end{table}

\textbf{Condorcet Winner Identification.}
Next, we consider the setting of identifying a Condorcet winner (CW) from a given strict preference probability matrix $P_{\succ} \in \mathcal{P}_{\succ}^{\mathrm{CW}}(n)$ for $n$ arms with an existing CW, where 
%
%\begin{align*}
%	
$$	\mathcal{P}_{\succ}^{\mathrm{CW}}(n) :=  \big\{   P_{\succ} \in [0,1]^{n\times n} \, \big| \, \exists \rho:  P_{\succ}^{(\rho,j)}>\frac12 , \forall j\neq \rho \big\} 
	\cap \mathcal{P}_{\succ}(n). $$
%	
%\end{align*}
%
Again, we distinguish between two classes of a (conventional) dueling bandits problem with different difficulties for this task similarly as above by defining 
$\mathcal{P}_i^{\mathrm{CW}}(n) := \mathcal{P}_i(n) \cap \mathcal{P}_{\succ}^{\mathrm{CW}}(n) $ for $i=1,2.$
%\quad \mbox{for} \ i=1,2.$
%
The experimental set-up (i.e., sampling a problem instance and repetition number) is similar as above.
%: For both classes, $\mathcal{P}_1^{\mathrm{CW}}(n)$ and $\mathcal{P}_2^{\mathrm{CW}}(n)$, an element is repeatedly sampled uniformly at random, and then passed to \Algo{POCOWISTA}, \Algo{TRA-POCOWISTA}, \Algo{DKWT}  and \Algo{SELECT} for a total of 100 repetitions.

Note that \Algo{POCOWISTA}, \Algo{TRA-POCOWISTA} and \Algo{DKWT} are parameter-free, while \Algo{SELECT} needs as a parameter the number of duels carried out per arm pair. 
In order to ensure that  \Algo{SELECT} fulfills the sought error probability $\delta$ this parameter needs to be chosen as a function of $\min_{i<j}|P_{\succ}^{(i,j)}-1/2|$ i.e., the \emph{unknown} minimal gap (in the conventional dueling bandit setting). 
In the following, we use the optimal choice for \Algo{SELECT}'s parameter, although this gives it a clear advantage.

The results for $n=20$ arms and different choices for the error probabilities $\delta$ for this experiment setting are reported in the last four columns of Table \ref{tab:COPE_IDENT_TAKS}.
%, i.e., the average sample complexity (and their standard deviation in brackets) of the four algorithms for the two problems classes.
%
Again, all algorithms have an empirical error probability of zero due to the Bonferroni correction.
The results show that \Algo{DKWT} is inferior to all other three algorithms, while \Algo{POCOWISTA} requires on average a slightly higher sample complexity than \Algo{SELECT} with the optimal choice of its parameter, and \Algo{TRA-POCOWISTA}'s sample complexity is the lowest. 
Nevertheless, the standard deviation of \Algo{SELECT}'s sample complexity is lower compared to our algorithms, which is due to the fact that the parameter of \Algo{SELECT} determines exactly how many pairwise comparisons are performed per pair of arms. 
The variance then arises from the different numbers of arm pairs used in total to arrive at the decision, which varies accordingly due to the random selection of the problem instance in each repetition.   
It is worth noting that the sampled problem instances do not necessarily satisfy the transitivity property in Definition \ref{def:transitivity}, so that the results are again in favor of \Algo{TRA-POCOWISTA}.

\section{CONCLUSION} \label{sec:conclusion}

In this paper, we considered an extension of the dueling bandits problem, where feedback in the form of an indifference can be observed in addition to the binary strict preference feedback. 
%
%This is motivated by a variety of practical application scenarios, particularly for cases, where preference feedback is provided by a human.
%
We have studied the pure exploration problem of finding a Copeland winner within a fixed confidence setting, for which we provided instance-dependent lower bounds on the sample complexity.    
Furthermore, we proposed \Algo{POCOWISTA}, which can solve this pure exploration task almost optimally for worst-case scenarios, and extended it to \Algo{TRA-POCOWISTA} for the case where the preference probabilities satisfy a certain type of stochastic transitivity that lead to improved sample complexity bounds.
%
%For its extension, called \Algo{TRA-POCOWISTA}, we derived an improved sample complexity with respect to the number of arms. 
%Furthermore, we proposed \Algo{POCOWISTA}, a learning algorithm that can solve this pure exploration task almost optimally, and extended it to the case where the preference probabilities satisfy a certain type of stochastic transitivity.
%
%For its extension, called \Algo{TRA-POCOWISTA}, we derived an improved sample complexity with respect to the number of arms. 

For future work, it would be interesting to investigate the considered extension of the dueling bandits problem in a regret minimization setting, or to combine it with other variants or extensions of the dueling bandits problem such as the multi-dueling setting \citep{saha2020pac,haddenhorst2021identification} or the non-stationary preference variants \citep{saha2022optimal,kolpaczki2022non,buening2022anaconda,suk2023can}.
Since our motivation for extending the dueling bandits problem stemmed from real-world examples, a more in-depth experimental study in one or more such practical application areas would certainly be a desirable avenue for future work, e.g., in algorithm configuration or online learning-to-rank problems for which preference-based bandit algorithms have been used before  \citep{BrSeCoLi16,ScOoWhDe16,OoScDe16,ZhKi16,el2020pool,DBLP:conf/aaai/BrandtSHBHT23}.  
In the course of this, it would be worth exploring contextualized variants of the preference-based bandit \citep{DBLP:conf/nips/Saha21,DBLP:conf/icml/BengsSH22,DBLP:conf/alt/SahaK22}.

%\subsubsection*{Acknowledgements}
%All acknowledgments go at the end of the paper, including thanks to reviewers who gave useful comments, to colleagues who contributed to the ideas, and to funding agencies and corporate sponsors that provided financial support. 
%To preserve the anonymity, please include acknowledgments \emph{only} in the camera-ready papers.

%\clearpage

%\bibliography{uai2023-template}
\bibliographystyle{icml2022}
\bibliography{references}

\clearpage

%\onecolumn
%\title{Identifying Copeland Winners in Dueling Bandits with Indifferences\\(Supplementary Material)}
%\maketitle
 
\appendix

\section{List of Symbols}

The following table contains a list of symbols that are frequently used in the main paper as well as in the following supplementary material. 
%\\ \medskip
%\small

%\begin{table}
%\centering
\resizebox{1.01\textwidth}{!}{
	\begin{tabular}{ll}
		\hline
		\multicolumn{2}{c}{\textbf{Basics}} \\
		\hline
		\\
		$\prec,\succ$ &  strict preference relation for objects, i.e., $o \succ o'$ (or $o \prec o'$ ) iff object $o$ is (not) preferred over object $o'$\\
		$\cong$ &  indifference relation for objects, i.e., $o \cong o'$ (or $o \cong o'$ ) iff object $o$ is not preferred over object $o'$ \\
		& and vice versa\\
		$1_{\lbrack \cdot \rbrack}$ & indicator function \\
		$\N$  & set of natural numbers (without 0), i.e., $\N = \{1,2,3,\dots\}$ \\
		$\R$ & set of real numbers\\
		$[n]$ & the set $\{1,2,\ldots,n\}$ for some $n\in \N$\\
		$p_{(1)},p_{(2)},p_{(3)}$ & order statistics of values $p_1,p_2,p_3,$ i.e., $p_{(1)}\geq p_{(2)} \geq p_{(3)}$ and $\{  p_{(1)},p_{(2)},p_{(3)} \} = \{ p_1,p_2,p_3 \}$\\
		$f_{\mathrm{Beta}}(x;a,b)$ & probability density function of the Beta distribution with parameters $a,b> 0$ at point $x \in \R$ \\
		\\
        $\KL(\mathbf{p},\mathbf{q})$ & Kullback-Leibler divergence for two categorical distributions \\
        & $\mathbf{p} =(p_1,\ldots,p_K) \in [0,1]^K$ and $\mathbf{q} =(q_1,\ldots,q_K)\in[0,1]^K $ such that $\sum_{i=1}^K p_i = \sum_{i=1}^K q_i =1$ \\
        $\mathrm{kl}(p,q)$ &  Kullback-Leibler divergence for two Bernoulli distributions with success probabilities  $p,q\in [0,1]$ \\
        & i.e, $\mathrm{kl}(p,q) =\KL((p,1-p),(q,1-q))$ \\
		\hline
		\multicolumn{2}{c}{\textbf{Modeling related}} \\
		\hline
		\\
		$n$ & number of arms\\
		$\mathcal{A} = [n]$ & set of arms  \\
		$P_{\succ}$ &  strict preference probability matrix with $P_{\succ}^{(i,j)}$ being the probability of observing $i \succ j$ \\
		 &		(element of $[0,1]^{n \times n}$)\\
		$P_{\prec}$ &  strict preference probability matrix with $P_{\prec}^{(i,j)}$ being the probability of observing $i \prec j$  \\ 
		&		(element of $[0,1]^{n \times n}$)\\
		$P_{\cong}$ &  indifference probability matrix with $P_{\cong}^{(i,j)}$ being the probability of observing $i \cong j$  \\ 
		& (element of $[0,1]^{n \times n}$)\\
		$P_{(1)}^{(i,j)}, P_{(2)}^{(i,j)} , P_{(3)}^{(i,j)}$ & order statistic  of $P_{\succ}^{(i,j)}, P_{\cong}^{(i,j)}$ and $P_{\prec}^{(i,j)},$ \\
		& i.e., $P_{(1)}^{(i,j)}\geq P_{(2)}^{(i,j)} \geq P_{(3)}^{(i,j)}$ and $\{P_{(1)}^{(i,j)}, P_{(2)}^{(i,j)} , P_{(3)}^{(i,j)} \} = \{ P_{\succ}^{(i,j)}, P_{\cong}^{(i,j)},P_{\prec}^{(i,j)} \}$ \\
		$\tern$ & probability distribution with probabilities  $P_{\succ}^{(i,j)}, P_{\cong}^{(i,j)}$ and $P_{\prec}^{(i,j)}$ for $i\succ j, i \cong j$ and $i \prec j$  \\
        $\mathbf{P}$ &  family of ternary distributions characterizing a dueling bandits problem instance with indifferences \\
        & i.e., $((P_{\succ}^{(i,j)},P_{\cong}^{(i,j)}, P_{\prec}^{(i,j)}))_{i<j}$ \\
		$\Delta(i,j)$	&  gap of the mode of $\tern,$ i.e.,  $	\Delta(i,j)  = | P_{(1)}^{(i,j)} - P_{(2)}^{(i,j)} |$ \\
		$\text{CP}(j)$ & Copeland score of an arm $j\in\mathcal{A}$ (see \eqref{def:Copeland_score}) \\
		$\mathcal{C}, \mathcal{C}(\mathbf{P})$  & Copeland set (see \eqref{def:Copeland_set}) resp.\ Copeland set for a given problem instance $\mathbf{P}$ \\
        %%
        %$I(j)$ & set of all indifferent arms to some arm $j\in\mathcal{A}$ \\
        %%
        %$L(j)$ & set of all superior arms to some arm $j\in\mathcal{A}$ \\
        %%
        %$W(j)$ & set of all inferior arms to some arm $j\in\mathcal{A}$ \\
        %
        $d_j$ & difference between the largest Copeland score and some arm $j$’s Copeland score \\
		\\
		\hline
		\multicolumn{2}{c}{\textbf{Learner related}} \\
		\hline
		\\
		$\Algo{A}$ & a learning for the dueling bandits problem with indifferences \\
		$(i_t,j_t)$ & pair of arms chosen by the learner in round $t$ (action of the learner in round $t$)\\
		$o_t$ & preference observation in round $t$ for the learner's action in round $t$\\
        & either $i_t \succ j_t,$ $i_t \cong j_t,$ or $i_t \prec j_t$ \\
		$\tau^{\Algo{A}}(\mathbf{P})$ &  sample complexity of the learning algorithm $\Algo{A},$  when started on a dueling bandits problem \\
        & with indifferences specified by $\mathbf{P}$\\
		$\hat{i}_{\Algo{A}}$ & Copeland Winner candidate returned by $\Algo{A}$ \\
		$\delta$ & specified probability of error (failure probability)  \\
		\\
		\hline
		\hline
				\multicolumn{2}{c}{\textbf{(TRA-)POCOWISTA related}} \\
				\hline
				\\
				$\widehat{CP}(i)$& current Copeland score of $i$ \\
				$\overline{CP}(i)$ & potential Copeland score of $i$ \\
				$D(i)$ & set of already compared arms to $i$ \\
				$W(i)$ & defeated arms by $i$ \\
				$I(i)$ & indifferent arms to $i$ \\
				$L(i) $ & superior arms to $i$ \\
				\\
				\hline
%				\multicolumn{2}{c}{\textbf{Abbreviations}} \\
%				\hline
%				\\
%				%	
%				CW & Condorcet Winner \\
%                %
%				COWI & Copeland Winner \\
%				%	
%				%
%				MAB & Multi-armed bandit \\
%				%	
%				PPR & prior-posterior-ratio \\
%				%	
%				\Algo{POCOWISTA} & \textbf{PO}tential \textbf{CO}peland \textbf{WI}nner \textbf{ST}ays \textbf{A}lgorithm \\
%    %
%				\Algo{TRA-POCOWISTA} & \textbf{TRA}nsitive \textbf{PO}tential \textbf{CO}peland \textbf{WI}nner \textbf{ST}ays \textbf{A}lgorithm \\
%				
%				&\hphantom{sample complexity of the learning algorithm $\Algo{A},$  when started on a dueling bandits problem with indifferences}\\	
%				\hline
		
	\end{tabular}
}

%\end{table}
\normalsize
%\flushleft
%\centering

\clearpage

\section{Proof of Theorem \ref{theorem:lower_bound_wo_ind_detailed_main_part}} \label{sec:proof_lower_bound}

% \textbf{Theorem 2.1}
% \emph{
% 	%	
% 	For any learning algorithm $\Algo{A}$ for the dueling bandits problem with indifferences described by an indifference probability matrix $\tilde P_{\cong}$ and strict preference probability matrix $\tilde P_{\succ},$ which satisfies  $\prob\big(\hat{i}_{\Algo{A}} \notin \mathcal{C}(\tilde P_{\succ},\tilde P_{\cong})\big) \leq \delta,$ where $\hat{i}_{\Algo{A}}$ is a Copeland Winner candidate returned by $\Algo{A},$
% 	%
% 	there exists a dueling bandits problem with indifferences instance characterized by $P_{\cong}$ and $P_{\succ},$ such that
% 	%
% 	$$ \mathbb{E}\big(  \tau^{\Algo{A}}(P_{\succ}) \big) \geq c \, \sum\nolimits_{i<j} \frac{\ln(1/\delta)}{\Delta(i,j)^2},$$
% 	%	
% 	where $c>0$ is some constant, $	\Delta(i,j)  = | P_{(1)}^{(i,j)} - P_{(2)}^{(i,j)} |$ and $P_{(1)}^{(i,j)}\geq P_{(2)}^{(i,j)} \geq P_{(3)}^{(i,j)}$ is the order statistic of $P_{\succ}^{(i,j)}, P_{\cong}^{(i,j)}$ and $P_{\prec}^{(i,j)}.$ 
% 	%
% 	%\begin{align*}
% 	%	
% 	%	\Delta(i,j)  = | P_{(1)} - P_{(2)} |
% 	%	 \Delta(i,j)  = \min\big\{ |P_{\succ}^{(i,j)} - &P_{\cong}^{(i,j)}|  , |P_{\succ}(j,i) - P_{\cong}^{(i,j)}|, \\
% 	%%	 
% 	%		 &|P_{\succ}(j,i) - P_{\succ}^{(i,j)}| \big\}. 
% 	%	
% 	%\end{align*}
% 	%
% 	In particular, if $\min_{i,j} \Delta(i,j) \geq \Delta$ for some $\Delta>0,$ then
% 	%
% 	$ \mathbb{E}\big(  \tau^{\Algo{A}}(P_{\succ}) \big) \in \Omega\left( \nicefrac{n^2 \ln(1/\delta)}{\Delta^2}  \right) .	$
% 	%
% }

Before giving the proof and discussion of Theorem \ref{theorem:lower_bound_wo_ind_detailed_main_part}, we need some additional notation and auxiliary results.
The \emph{Kullback-Leibler divergence} for two categorical distributions $\mathbf{p} =(p_1,\ldots,p_K) \in [0,1]^K$ and $\mathbf{q} =(q_1,\ldots,q_K)\in[0,1]^K $ such that $\sum_{i=1}^K p_i = \sum_{i=1}^K q_i =1$ is given by
\begin{equation*}
	\KL(\mathbf{p},\mathbf{q}) 
	= \begin{cases}  \sum\nolimits_{i \in [K]: p_i > 0} p_i \ln \left(\frac{p_i}{q_i} \right), \, &\text{if } \forall j \in [K]: q_j = 0 \Rightarrow p_j = 0, \\  
	\infty, &\text{otherwise.}\end{cases}
\end{equation*}
If $K=2,$ we will simply write $\mathrm{kl}(p,q) \coloneqq \KL((p,1-p),(q,1-q))$ for any $p,q\in [0,1]$.
\begin{lemma}\label{Le_Le1_for_LB_Proof}
	\begin{itemize}
		\item[(i)]
		For any two categorical distributions $\mathbf{p} =(p_1,\ldots,p_K) \in [0,1]^K$ and $\mathbf{q} =(q_1,\ldots,q_K)\in[0,1]^K,$ it holds that
		\begin{equation*}
			\KL(\mathbf{p},\mathbf{q}) \leq \sum\nolimits_{i=1}^K \frac{(p_i-q_i)^{2}}{q_i}.
		\end{equation*}
		\item[(ii)]
		For any $\delta \in (0,1)$ it holds that $\mathrm{kl}(\delta,1-\delta) \geq \ln((2.4\delta)^{-1}).$ 
		\item[(iii)]
		For any $p,q \in[0,1]$ it holds that $2(p-q)^2 \leq \mathrm{kl}(p,q) \leq \frac{(p-q)^2}{q(1-q)} .$ 
	\end{itemize}
\end{lemma}
%
%\begin{proof}
%%	
%	The statement from (i) is Lemma 3 in \citep{Chen2018} and for (ii) cf. Equation (3) in \citep{kaufmann2016complexity}.
%%
%\end{proof}

For any learning algorithm $\Algo{A}$ for the dueling bandits problem with indifferences let $\tau^{\Algo{A}}(\mathbf{P})$ denote its number of samples, when started on a dueling bandits problem with indifferences specified by $\mathbf{P} = ((P_{\succ}^{(i,j)},P_{\cong}^{(i,j)}, P_{\prec}^{(i,j)}))_{i<j}$
Further, let us write $d^{\Algo{A}}_{t}$ for the duel  (element of $\mathcal{A}  \times \mathcal{A}$) made at time step $t$.
Define $\tau_{i,j}^{\Algo{A}}(\mathbf{P})$ to be the number of times $\Algo{A}$ compares $(i,j)$ or equivalently $(j,i)$ before termination, i.e., 
$$\tau_{i,j}^{\Algo{A}}(\mathbf{P}) = \sum_{t=1}^{\tau^{\Algo{A}}(\mathbf{P})} \IND{ d^{\Algo{A}}_{t} = \{i,j\} },$$ 
so that $\tau^{\Algo{A}}(\mathbf{P}) = \sum_{ i <j } \tau_{i,j}^{\Algo{A}}(\mathbf{P}).$  
In the following, we will simply write $\tau^{\Algo{A}}$ for $\tau^{\Algo{A}}(\mathbf{P})$ and $\tau_{i,j}^{\Algo{A}}$ for $\tau_{i,j}^{\Algo{A}}(\mathbf{P}).$

Let $o^{\Algo{A}}_{t}$ be the feedback observed by $\Algo{A}$ at time step $t$, after conducting the duel $d^{\Algo{A}}_{t}.$ 
Write $\mathcal{F}^{\Algo{A}}_{t} = \sigma(d^{\Algo{A}}_{1},o^{\Algo{A}}_{1}, \dots, d^{\Algo{A}}_{t},o^{\Algo{A}}_{t})$ for the sigma algebra generated by the choices and the corresponding observed feedback of $\Algo{A}$ until time $t$, and as usual
$\mathcal{F}_{\tau^{\Algo{A}}} = \{B \in \sigma(\bigcup_t \mathcal{F}_t^{\Algo{A}}) : B \cap \{T^{\mathcal{A}} \leq t\} \in \mathcal{F}_{t}^{\Algo{A}} \forall t\in \N \}$.
% $\mathcal{F}_{\tau^{\Algo{A}}} \coloneqq \mathcal{F}_{\tau^{\Algo{A}}}^{\Algo{A}} = \sigma \left( \bigcup_{t\leq \tau^{\Algo{A}}} \mathcal{F}_{t}^{\Algo{A}} \right)$.  
%
Note that $\Algo{A}$ can be interpreted as a learning algorithm for the classical multi-armed bandit with $\binom{n}{2}$ many arms (one for each possible pair) and ``rewards'' $r(o^{\Algo{A}}_{t}) \in \{-1,0,1\}$, where for $o_t^{\Algo{A}} \in \{   i_t \succ j_t, i_t \cong j_t, i_t \prec j_t\}$
$$  r(o_t^{\Algo{A}}) = \begin{cases}
	-1, & o_t^{\Algo{A}} = i_t \prec j_t, \\
	0, & o_t^{\Algo{A}} = i_t \cong j_t, \\
	1, & o_t^{\Algo{A}} = i_t \succ j_t. \\
\end{cases} $$
For sake of convenience, let us write $\tern = (P_{\succ}^{(i,j)},P_{\cong}^{(i,j)},P_{\prec}^{(i,j)})^\top,$ so that $\mathbf{P} = (\tern)_{i<j}.$ 
%for the set of ternary distributions of the dueling bandits problem with indifferences, i.e., each $\tern$ is a ternary distribution (categorical distribution with three categories) with probabilities $P_{\succ}^{(i,j)},P_{\cong}^{(i,j)}$ and $P_{\prec}^{(i,j)}$ for $i\succ j, i \cong j$ and $i \prec j.$
%
With this, we may transfer Lemma 1 from \citep{kaufmann2016complexity} to our setting as follows:
\begin{lemma}\label{Le_from_Kaufmann}
	Let $\mathbf{P}= (P_{\succ}^{(i,j)},P_{\cong}^{(i,j)},P_{\prec}^{(i,j)})_{i<j}$, $\widetilde{\mathbf{P}}= (\widetilde{P}_{\succ}^{(i,j)},\widetilde{P}_{\cong}^{(i,j)}, \widetilde{P}_{\prec}^{(i,j)})_{i<j}$ be two problem instances of the dueling bandits problem with indifferences such that $\min_{i<j}  \tern>0$ and $\min_{i<j}  \terntilde > 0.$
	For any learning algorithm $\Algo{A}$ for the dueling bandits problem with indifferences, which fulfills $\mathbb{E}_{\mathbf{P}}[\tau^{\Algo{A}}(\mathbf{P})],\mathbb{E}_{\widetilde{\mathbf{P}}}[\tau^{\Algo{A}}(\widetilde{\mathbf{P}})]<\infty$, it holds that
	\begin{equation*}
		\sum\nolimits_{i < j} \mathbb{E}_{\mathbf{P}}\left[ \tau_{i,j}^{\Algo{A}}(\mathbf{P}) \right] \KL( \tern , \terntilde ) \geq 
		\sup\nolimits_{\mathcal{E} \in \mathcal{F}_{\tau^{\Algo{A}}}} \mathrm{kl}\left(\mathbb{P}_{\mathbf{P}}(\mathcal{E}),\mathbb{P}_{\mathbf{P'}}(\mathcal{E})\right).
	\end{equation*}
 In case $\mathbf{P}$ and $\mathbf{\widetilde{P}}$ do not have indifferences (i.e., if $\max_{i,j} P_{\cong}^{(i,j)} = 0 = \max_{i,j} \widetilde{P}_{\cong}^{(i,j)}$), the same inequality holds with $\KL(P^{(i,j)},\widetilde{P}^{(i,j)})$ replaced by $\mathrm{kl}(P_{\succ}^{(i,j)},\widetilde{P}_{\succ}^{(i,j)})$, respectively.
\end{lemma}

We are now ready to prove Theorem \ref{theorem:lower_bound_wo_ind_detailed_main_part}.

Recall $L(j)$ and $I(j)$, let $d_j = CP^\ast - CP(j)$ as above. In order to capture the dependence on the underlying instance  $\mathbf{P}$ of interest, we may simply write $L(\mathbf{P},j)$ and $I(\mathbf{P},j)$ as well as $d_{j}(\mathbf{P})$, $\text{CP}^{\ast}(\mathbf{P})$ and $\text{CP}(\mathbf{P},j)$ for these corresponding terms. For any set $X\subseteq [n]$ and $k\in \N$, let us denote by $X^{[k]}$ the set of $k$-sized subsets of $X$.
If $\mathbf{P}$ has no indifferences, let us simply write 
\begin{equation*}
    \kappa_{x,y}(\mathbf{P}) \coloneqq \mathrm{kl}\left( P_{\succ}^{(x,y)}, 1- P_{\succ}^{(x,y)}\right)
\end{equation*}
for any $x,y\in \mathcal{A}.$
%
% Furthermore, let us write $\text{CP}(\mathbf{P},j)$ for the Copeland score of arm $j$ on instance $\mathbf{P}$.

\begin{theorem}\label{theorem:lower_bound_wo_ind_detailed}
    If $\Algo{A}$ is an algorithm that correctly identifies the Copeland winner with confidence $1-\delta$ for any $\mathbf{P}$ without indifferences, then we have for all such $\mathbf{P}$ with $\min_{i<j} \min\{P_{\succ}^{(i,j)}, 1- P_{\succ}^{(i,j)}\} > 0$ and $\mathcal{C}(\mathbf{P}) = \{i^{\ast}\}$ (for some $i^\ast$, i.e., there is a unique Copeland winner), then 
    \begin{align}\label{eq:thm_lower_bound_wo_ind_detailed}
        \begin{split}
                \exptd_{\mathbf{P}}[\tau^{\Algo{A}}] 
        \geq \ln \frac{1}{2.4\delta} \sum_{j\in \mathcal{A} \setminus \{i^\ast\}} &\max \left\{ \frac{|L(j)|}{d_j +1 } \IND{ |L(j)| \geq d_j + 1 } , \frac{|L(j)|-1}{|L(j)| + d_j - 2} \IND{ i^\ast \in L(j) } \right\} \\
        &\cdot \ \min_{z\in L(j)} \frac{1}{\kappa_{j,z}(\mathbf{P})}.
        \end{split}
    \end{align}
\end{theorem}

Note that the right-hand side of \eqref{eq:thm_lower_bound_wo_ind_detailed} depends not only via $\kappa_{j,z}(\mathbf{P})$ but also via $d_{j} = d_{j}(\mathbf{P})$, $L(j) = L(\mathbf{P},j)$ on the underlying instance $\mathbf{P}$.  Before stating its proof, the following lemma assures us that the lower bound from Theorem~\ref{theorem:lower_bound_wo_ind_detailed} is in any case non-trivial.
 \begin{lemma}\label{lemma:lower_bound_1}
    Suppose $\mathbf{P}$ has no indifferences  and 
    $\mathcal{C}(\mathbf{P}) = \{i^\ast\}$ holds and let 
    $j\in \mathcal{A} \setminus \{i^\ast\}$ be arbitrary. 
    If $i^\ast \not \in L(j)$, then $|L(j)| \geq d_j + 1$.
\end{lemma} 
\begin{proof}[Proof of Lemma~\ref{lemma:lower_bound_1}]
    As $\mathbf{P}$ has no indifferences, the Copeland scores are given as $$\text{CP}(\mathbf{P},l) = (n-1) - |L(l)|$$ for any $l\in \mathcal{A}$. If $i^\ast \not\in L(j)$, then $j\in L(i^\ast),$ which implies  $\text{CP}(\mathbf{P},i^\ast)  \leq (n-1) - 1 = n-2$ and thus  
    \begin{equation*}
        d_j = \text{CP}(\mathbf{P},i^\ast) - \text{CP}(\mathbf{P},j) \leq (n-2) - ((n-1) - |L(j)|) = |L(j)| -1 
    \end{equation*}
    follows directly.
\end{proof}
According to the previous lemma, for any instance $\mathbf{P}$, at least one of the indicators appearing in \eqref{eq:thm_lower_bound_wo_ind_detailed} is $1$, whence the lower bound is larger than $0$. In case $i\ast \in L(j)$ and $L(j)$ is large, it is possible that both indicators are $1$. We proceed with the proof of the theorem. 

\begin{proof}[Proof of Theorem~\ref{theorem:lower_bound_wo_ind_detailed}]
    Suppose $\Algo{A}$ and $\mathbf{P}$ with $\mathcal{C}(\mathbf{P}) = \{i^\ast\}$ without indifferences are fixed.\\[0.5em]
    \noindent \textbf{Claim~1:} The following holds:
    \begin{itemize}
        \item[(i)] If $\mathcal{L} \subseteq L(j)$ fulfills $|\mathcal{L}| \geq d_j+1$, then 
        \begin{equation}\label{eq:proof_indiff_lb_claim1}
            \ln \frac{1}{2.4\delta} \leq \sum\nolimits_{z\in \mathcal{L}} \exptd_{\mathbf{P}}[\tau_{jz}^{\Algo{A}}] \kappa_{j,z}(\mathbf{P}).
        \end{equation}
        \item[(ii)] If $i^\ast \in L(j)$ and $\mathcal{L} \subseteq L(j) \setminus \{i^\ast\}$ fulfills $|\mathcal{L}| \geq d_j-1$, then \eqref{eq:proof_indiff_lb_claim1} holds as well.   
    \end{itemize}
    \noindent \textbf{Proof of Claim~1:} To prove (i), suppose $\mathcal{L} \subseteq L(j)$ with $|\mathcal{L}| \geq d_j + 1$ to be arbitrary but fixed for the moment. Define the instance $\mathbf{\widetilde{P}}$ via 
    \begin{equation*}
        \widetilde{P}_{\succ}^{(x,y)} \coloneqq \begin{cases}1- P_{\succ}^{(x,y)},\quad &\text{if } (x,y)\in \{(j,z),(z,j)\} \text{ for } z\in \mathcal{L},\\
        P_{\succ}^{(x,y)}, \quad &\text{otherwise}.\end{cases}
    \end{equation*} 
    By construction we have $\text{CP}(\mathbf{\widetilde{P}},i^\ast) \leq \text{CP}(\mathbf{P},i^\ast)$ and obtain 
    \begin{align*}
        \text{CP}(\mathbf{\widetilde{P}},j) &= \text{CP}(\mathbf{P},j) + |\mathcal{L}|  \geq \text{CP}(\mathbf{P},j) + d_j + 1 \\
        &= \text{CP}(\mathbf{P},j) + (\text{CP}(\mathbf{P},i^\ast) - \text{CP}(\mathbf{P},j)) + 1 = \text{CP}(\mathbf{P},i^\ast) + 1 \geq \text{CP}(\mathbf{\widetilde{P}},i^\ast) + 1.
    \end{align*}
    This shows $i^\ast \not\in \mathcal{C}(\mathbf{\widetilde{P}})$, and by assumption on $\Algo{A}$ the event $\mathcal{E} \coloneqq \{i^\ast \in \mathcal{C}(\mathbf{P})\} \in \mathcal{F}_{\tau^\Algo{A}}$ has the properties
    \begin{equation*}
        \prob_{\mathbf{P}}(\mathcal{E}) \geq 1-\delta \quad \text{and} \quad \prob_{\mathbf{\widetilde{P}}}(\mathcal{E}) \leq \delta.
    \end{equation*}
    Therefore, Lemma~\ref{Le_from_Kaufmann} and part~(ii) of Lemma~\ref{Le_Le1_for_LB_Proof} assure
    \begin{equation}\label{eq:proof_LB_wo_indiff_claim1_eq}
        \sum\nolimits_{x<y} \exptd_{\mathbf{P}} \left[ \tau_{x,y}^{\Algo{A}} \right] \KL\left(P_{x,y}, \widetilde{P}_{x,y} \right) \geq \mathrm{kl}\left(\prob_{P_\succ}(\mathcal{E}),\prob_{\widetilde{P}_\succ}(\mathcal{E})\right) \geq \mathrm{kl}(1-\delta,\delta) \geq \ln \frac{1}{2.4\delta} .
    \end{equation}
    In case $(x,y) \not\in \{(j,z),(z,j)\}$ for any $z\in \mathcal{L}$, it holds that $\widetilde{P}_{x,y} = P_{x,y}$ so that $\KL\left(P_{x,y}, \widetilde{P}_{x,y} \right) = 0$. In case $z\in \mathcal{L} \subseteq L(j)$ we have 
    \begin{equation*}
        \KL\left(P_{j,z}, \widetilde{P}_{j,z} \right) = \KL\left((P_{\succ}^{(j,z)},P_{\prec}^{(j,z)}), (P_{\prec}^{(j,z)},P_{\succ}^{(j,z)}) \right)  = \mathrm{kl}\left( P_{\succ}^{(j,z)}, 1- P_{\succ}^{(j,z)}\right) = \kappa_{j,z}(\mathbf{P}).
    \end{equation*}
    Combining these estimates with \eqref{eq:proof_LB_wo_indiff_claim1_eq} proves \eqref{eq:proof_indiff_lb_claim1}.\\[0.5em]
    To prove (ii) suppose $i^\ast \in L(j)$ holds and $\mathcal{L} \subseteq L(j) \setminus \{i^\ast\}$ fulfills $|\mathcal{L}| \geq d_j-1$. Define the instance $\mathbf{\widetilde{P}}$ via 
    \begin{equation*}
        \widetilde{P}_{\succ}^{(x,y)} \coloneqq \begin{cases}1- P_{\succ}^{(x,y)},\quad &\text{if } (x,y)\in \{(j,z),(z,j)\} \text{ for } z\in \mathcal{L} \cup \{i^\ast\}\\
        P_{\succ}^{(x,y)}, \quad &\text{otherwise}.\end{cases}
    \end{equation*} 
    Regarding that $\widetilde{P}_{\succ}^{(j,i^\ast)} = 1- P_{\succ}^{(j,i^\ast)}$ and  $\widetilde{P}_{\succ}^{(z,i^\ast)} =  P_{\succ}^{(z,i^\ast)}$ for $z\not= j$, we have $\text{CP}(\mathbf{\widetilde{P}},i^\ast) = \text{CP}(\mathbf{P},i^\ast) -1$, and similarly we see $\text{CP}(\mathbf{\widetilde{P}},j) = \text{CP}(\mathbf{P},j) + |\mathcal{L}|+1$. Together with $|\mathcal{L}| \geq d_j -1$, we obtain with the same argumentation as before that $\text{CP}(\mathbf{\widetilde{P}},j) \geq \text{CP}(\mathbf{\widetilde{P}},i^\ast) +1$ and thus $i^\ast \not\in \mathcal{C}(\mathbf{\widetilde{P}})$. Therefore, following the lines from above, we conclude that \eqref{eq:proof_indiff_lb_claim1} also holds in this case.  $\blacksquare$\\[0.5em]
    To prove the theorem, abbreviate for convenience $\kappa_{x,y} = \kappa_{x,y}(\mathbf{P})$ in the following, and let us at first suppose that $|L(j)|\geq d_j + 1$ holds. When summing the inequality \eqref{eq:proof_indiff_lb_claim1} over all $\binom{|L(j)|}{d_j +1}$ many $\mathcal{L} \subseteq L(j)$ of size $|\mathcal{L}| = d_j +1$, any of the summands $\exptd_{\mathbf{P}}[\tau_{jz}^{\Algo{A}}] \kappa_{j,z}$, with $z\in L(j)$, appears exactly $\binom{|L(j)|}{d_j +1}$ times, i.e., we have 
    \begin{align*}
        \binom{|L(j)|}{d_j +1} \ln \frac{1}{2.4\delta} &\leq \sum\nolimits_{z\in \mathcal{L}} \binom{|L(j)|-1}{d_j}  \exptd_{\mathbf{P}}[\tau_{jz}^{\Algo{A}}] \kappa_{j,z} \\
        &\leq \binom{|L(j)|-1}{d_j}\left( \max\nolimits_{z\in L(j)} \kappa_{j,z}\right) \sum\nolimits_{z\in L(j)} \exptd_{\mathbf{P}}[\tau_{jz}^{\Algo{A}}].
    \end{align*}
    Using that $\binom{a}{b} / \binom{a-1}{b-1} = \frac{a}{b}$ holds for any $a,b\in \N$ with $a\leq b$, we infer
    \begin{equation*}
        \exptd_{\mathbf{P}}\left[\tau_j^{\Algo{A}}\right] \geq \sum\nolimits_{z\in L(j)} \exptd_{\mathbf{P}} \left[ \tau_{jz}^{\Algo{A}}\right] \geq \left( \ln \frac{1}{2.4\delta} \right) \frac{|L(j)|}{d_j +1} \min\nolimits_{z\in L(j)} \frac{1}{\kappa_{j,z}}.
    \end{equation*}
    Now, suppose $i^\ast \in L(j)$. When summing \eqref{eq:proof_indiff_lb_claim1} over all $\binom{|L(j)|-1}{d_j -1}$ many $\mathcal{L} \subseteq L(j) \setminus \{i^\ast\}$ with $|\mathcal{L}| = d_j - 1$, we observe 
    $\exptd_{\mathbf{P}}[\tau_{ji^\ast}^{\Algo{A}}] \kappa_{j,i^\ast}$ exactly $\binom{|L(j)|-1}{d_j -1}$ times as a summand and each of the terms $\exptd_{\mathbf{P}}[\tau_{jz}^{\Algo{A}}] \kappa_{j,z}$, with $z\in L(j)$, exactly\footnote{Note here that $\mathcal{L} = \emptyset$ if $d_{j} = 1$.} $\binom{|L(j)|-2}{d_j -2}\IND{d_j \geq 2 }$ times as a summand. Therefore, we get 
    \begin{align}
        \begin{split} \label{eq:proof__lower_bound_wo_ind_detailed} 
                \binom{|L(j)|-1}{d_j -1 } \ln \frac{1}{2.4\delta} &\leq \binom{|L(j)|-1}{d_j -1} \exptd_{\mathbf{P}}[\tau_{ji^\ast}^{\Algo{A}}] \kappa_{j,i^\ast}  \\ 
                &\quad + \sum\nolimits_{z\in \mathcal{L}} \binom{|L(j)|-2}{d_j-2}  \exptd_{\mathbf{P}}[\tau_{jz}^{\Algo{A}}] \kappa_{j,z} \IND{ d_j \geq 2 } .
        \end{split}
        % \\
        % &\leq \left( \binom{|L(j)|-1}{d_j-1} + \binom{|L(j)|-2}{d_j-2} \right) \left( \max\nolimits_{z\in L(j)} \kappa_{j,z}\right) \sum\nolimits_{z\in L(j)} \exptd_{\mathbf{P}}[\tau_{jz}^{\Algo{A}}] \bold{1}_{\{d_j \geq 2\}}. 
    \end{align}
    If $d_j \geq 2$, we obtain from 
    \begin{align*}
        \binom{|L(j)|-1}{d_j-1} + \binom{|L(j)|-2}{d_j-2} &=  \binom{|L(j)|-1}{d_j-1}  + \frac{d_j-1}{|L(j)|-1} \binom{|L(j)|-1}{d_j-1} \\  
        &= \frac{|L(j)|+d_j-2}{|L(j)|-1} \binom{|L(j)|-2}{d_j-2}
    \end{align*}
    that 
    \begin{align*}
    \exptd_{\mathbf{P}}[\tau_{j}^{\Algo{A}}]  &\geq \exptd_{\mathbf{P}}[\tau_{jz}^{\Algo{A}}]  \geq \ln \frac{1}{2.4\delta} \frac{|L(j)|-1}{|L(j)|+d_j-2} \left( \min\nolimits_{z\in L(j)} \frac{1}{\kappa_{j,z}}\right)
    \end{align*}
    and \eqref{eq:thm_lower_bound_wo_ind_detailed} follows.
    In the other case $d_j=1$, we have $\frac{|L(j)|+d_j-2}{|L(j)|-1} = 1$ and thus \eqref{eq:thm_lower_bound_wo_ind_detailed} can be inferred from \eqref{eq:proof__lower_bound_wo_ind_detailed}. This completes the proof.
\end{proof}

Next, we want to prove an analogon of the above lower bound for the more sophisticated scenario of dueling bandits with indifferences. To prepare this, 
define for any instance $\mathbf{P}$ with indifferences and $x,y\in \mathcal{A}$ the terms
\begin{align*}
    D_{x,y}(\mathbf{P}) \coloneqq \max \Big\{ 
    &\KL\left( \left( P_{\succ}^{(x,y)},P_{\cong}^{(x,y)},P_{\prec}^{(x,y)}\right), \left( P_{\cong}^{(x,y)},P_{\succ}^{(x,y)},P_{\prec}^{(x,y)} \right) \right), \\
        &\KL\left( \left( P_{\succ}^{(x,y)},P_{\cong}^{(x,y)},P_{\prec}^{(x,y)}\right), \left( P_{\prec}^{(x,y)},P_{\cong}^{(x,y)},P_{\succ}^{(x,y)}\right) \right)
        \Big\}.
\end{align*}
If $\mathbf{P}$ is fixed or clear by the context, we may simply write $D_{x,y}$ instead of $D_{x,y}(\mathbf{P})$.

\begin{theorem}\label{theorem:lower_bound_detailed}
    If $\Algo{A}$ is an algorithm that correctly identifies the Copeland winner with confidence $1-\delta$ for any $\mathbf{P}$, then we have for all $\mathbf{P}$ with $\min_{i<j} \min\{P_{\succ}^{(i,j)},P_{\cong}^{(i,j)},P_{\prec}^{(i,j)}\}>0$ and  $\mathcal{C}(\mathbf{P}) = \{i^{\ast}\}$ (for some $i^\ast$, i.e., there is a unique Copeland winner) the bound 
    \begin{align*}
        \mathbb{E}_{\mathbf{P}}[\tau^{\Algo{A}}] \geq \ln \frac{1}{2.4\delta}\sum\nolimits_{j \in \mathcal{A} \setminus \{i^\ast\}} \max \left\{ C_j, C'_j \IND{i^\ast \in I(j) }, C''_j \IND{i^\ast \in L(j) } \right\}\min_{z\in L(j) \cup I(j)} \frac{1}{D_{j,z}(\mathbf{P})}
    \end{align*}
    where 
    \begin{align*}
        C_j &\coloneqq \max_{(i,l) \in \Psi(j)} \frac{\binom{|I(j)|}{i} \binom{|L(j)|}{l}}{\binom{|I(j)|-1}{i-1} \binom{|L(j)|}{l} \IND{i\geq 1} + \binom{|I(j)|}{i} \binom{|L(j)|-1}{l-1} \IND{l\geq 1 }},\\
        C'_j&\coloneqq  \max_{(i,l) \in \Psi'(j)}  \frac{\binom{|I(j)|-1}{i} \binom{|L(j)|}{l}}{\binom{|I(j)|-1}{i} \binom{|L(j)|}{l} +   \binom{|I(j)|-2}{i-1} \binom{|L(j)|}{l} \IND{i\geq 1 } +  \binom{|I(j)|-1}{i} \binom{|L(j)|-1}{l-1} \IND{l\geq 1 }},\\
        C''_j &\coloneqq \max_{(i,l) \in \Psi''(j)} \frac{\binom{|I(j)|}{i} \binom{|L(j)|-1}{l}}{\binom{|I(j)|}{i} \binom{|L(j)|-1}{l} +   \binom{|I(j)|-1}{i-1} \binom{|L(j)|-1}{l} \IND{ i\geq 1 } +  \binom{|I(j)|}{i} \binom{|L(j)|-2}{l-1} \IND{l\geq 1 }}
    \end{align*}
    with 
    \begin{align*}
        \Psi(j) &\coloneqq  \left\{ (i,l) \, : \, i \in \{0,\dots,|I(j)|\}, l\in \{0,\dots,|L(j)|\} \text{ and } i+2l \geq 2d_j+1 \right\},\\
        \Psi'(j) &\coloneqq  \left\{ (i,l) \, : \, i \in \{0,\dots,|I(j)|-1\}, l\in \{0,\dots,|L(j)|\} \text{ and } i+2l \geq 2d_j -1 \right\},\\
        \Psi''(j) &\coloneqq  \left\{ (i,l) \, : \, i \in \{0,\dots,|I(j)|\}, l\in \{0,\dots,|L(j)|-1\} \text{ and } i+2l \geq 2d_j-3 \right\}.
    \end{align*}
\end{theorem}

Before providing its proof, let us briefly discuss this lower bound. At first, note that in fact all of the binomial coefficients appearing in the definitions of $C_j$, $C'_j$ and $C''_j$ are well-defined and of the form $\binom{n}{k}$ for $0\leq k\leq n$. In the definition of $C_j$ this is assured by means of the indicator functions $\IND{i\geq 1 }$ and $\IND{l\geq 1 }$, and in the definition of $C'_j$ resp. $C''_j$ this follows from the definitions of $\Psi'(j)$ resp. $\Psi''(j)$. For example, if $(i,l) \in \Psi'(j)$, then $i\geq 0$ assures $|I(j)| \geq i+1 \geq 1$, and in case $i\geq 1$ we have $|I(j)| \geq 2$, whence $\binom{|I(j)|-2}{i-1} \binom{|L(j)|}{l}$ is well-defined.\par 
    Note that Thm.~\ref{theorem:lower_bound_wo_ind_detailed_main_part} is a direct consequence of the stated lower bounds. For the sake of completeness, we formulate it as follows.
    \begin{proof}[Proof of Theorem~\ref{theorem:lower_bound_wo_ind_detailed_main_part}]
    Part (i) is exactly Theorem~\ref{theorem:lower_bound_wo_ind_detailed}, and part (ii) follows due to $\max\{C_j,C'_j\IND{i^\ast \in I(j)}, C''_j \IND{i^\ast \in L(j)}\} \geq C_j$ directly from Theorem~\ref{theorem:lower_bound_detailed}.
    \end{proof}

The maximum in the proof of Thm.~\ref{theorem:lower_bound_detailed} actually assures that the lower bound is non-trivial on any instance $\mathbf{P}$ considered in the theorem. This is made formal in the upcoming lemma.
\begin{lemma}\label{lemma:B5_is_never_trivial}
    Let $\mathbf{P}$ be an instance with indifferences such that $\mathcal{C}(\mathbf{P}) = \{i^\ast\}$. Then, for any $j\not= i^\ast$ exactly one of the following holds:
    \begin{itemize}
        \item[(i)] $i^\ast \not\in L(j) \cup I(j)$ and $\Psi(j)\not=\emptyset$,
        \item[(ii)] $i^\ast \in I(j)$ and $\Psi'(j)\not=\emptyset$,
        \item[(iii)] $i^\ast \in L(j)$ and $\Psi''(j)\not=\emptyset$.
    \end{itemize}
\end{lemma}
\begin{proof}[Proof of Lemma~\ref{lemma:B5_is_never_trivial}]
    Suppose $\mathbf{P}$ with $\mathcal{C}(\mathbf{P}) = \{i^\ast\}$ and $j\not= i^\ast$ to be fixed. Regarding the definition of the Copeland score, we have
    \begin{equation*}
        \text{CP}(\mathbf{P},l) = (n-1) - |L(l)| - \frac{1}{2}|I(l)|
    \end{equation*}
    for any $l\in \mathcal{A}$. For fixed $j\not= i^\ast$ we have in particular 
    \begin{equation*}
        \text{CP}(\mathbf{P},i^\ast) \leq \begin{cases} n-2, &\text{ if } i^\ast \not\in I(j) \cup L(j),\\
        n-3/2, &\text{ if } i^\ast \in I(j),\\
        n-1, &\text{ if } i^\ast \in L(j)
        \end{cases}
    \end{equation*}
    and thus 
    \begin{align*}
        d_j = d_{j}(\mathbf{P}) \leq \begin{cases}
            |L(j)| + |I(j)|/2 - 1, &\text{ if } i^\ast \not\in L(j) \cup I(j),\\
            |L(j)| + |I(j)|/2 - 1/2, &\text{ if } i^\ast \in I(j),\\
            |L(j)| + |I(j)|/2, &\text{ if } i^\ast \in L(j).
        \end{cases}
    \end{align*}
    If $i^\ast \not\in L(j) \cup I(j)$, then $2d_j \leq 2|L(j)| + |I(j)| -2$ and $(|I(j)|,|L(j)|) \in \Psi(j)$ follows. In case $i^\ast \in I(j)$, $2d_j \leq 2|L(j)| + |I(j)|-1$ and thus $(|I(j)|-1,|L(j)|) \in \Psi'(j)$, and similarly we see in case $i^\ast \in L(j)$ that $2d_j \leq 2|L(j)| + |I(j)|$ implies $(|I(j)|,|L(j)|-1) \in \Psi''(j)$.
\end{proof}

In contrast to Thm.~\ref{theorem:lower_bound_detailed}, the corresponding simplified version stated in  Thm.~\ref{theorem:lower_bound_wo_ind_detailed_main_part} is e.g. trivial on particular instance $\mathbf{P}$ defined via
\begin{equation*}
    (P_{\succ}^{(i,j)})_{i,j} = \begin{pmatrix} - & 1/2 & 1/2 \\ 1/4 & - & 1/2 \\
    1/4 & 1/4 & -\end{pmatrix}.
\end{equation*}
To see this, note that $P_{\cong}^{(x,y)} = 1/4$ holds for all $x,y \in \mathcal{A}$, $\mathcal{C}(\mathbf{P}) = \{1\}$ and observe that $d_2=1 = |L(2)|$ resp. $d_3 = 2 = |L(3)|$ and $|I(2)| = |I(3)| = 0$ imply $\Psi(2) = \emptyset$ resp. $\Psi(3) = \emptyset$.\par
Now, let us proceed with the proof of Thm.~\ref{theorem:lower_bound_detailed}. The proof idea is similar to that of Thm.~\ref{theorem:lower_bound_detailed}, but as it is more sophisticated and technical, we prove it for the sake of completeness in detail.

\begin{proof}[Proof of Theorem \ref{theorem:lower_bound_detailed}]
    Suppose $\Algo{A}$ and $\mathbf{P}$ with $\mathcal{C}(\mathbf{P}) = \{i^\ast\}$ are fixed. Assume w.l.o.g. $i^\ast = 1$ and let $j\in [n] \setminus \{i^{\ast}\}$ be arbitrary but fixed for the moment. \\[0.5em] 
    \textbf{Claim~1:} The following holds:
    \begin{itemize}
        \item[(i)] If $(i,l) \in \Psi(j)$,  we have for each $\mathcal{I} \in  I(j)^{[i]}, \mathcal{L} \in  L(j)^{[l]}$ that   
        \begin{equation}\label{eq_LBindiff_Cl1b}
            \ln \frac{1}{2.4\delta} \leq \sum\nolimits_{z\in \mathcal{I}} \exptd_{\mathbf{P}}\left[\tau_{jz}^{\Algo{A}}\right] D_{j,z}(\mathbf{P}) + \sum\nolimits_{z\in \mathcal{L}} \exptd_{\mathbf{P}}\left[\tau_{jz}^{\Algo{A}}\right] D_{j,z}(\mathbf{P}).
        \end{equation}
        \item[(ii)] If $(i,l) \in \Psi'(j)$ and $i^\ast \in I(j)$, \eqref{eq_LBindiff_Cl1b} holds for all $\mathcal{I} \in  I(j)^{[i]}, \mathcal{L} \in  L(j)^{[l]}$ with $i^\ast \in \mathcal{I}$.
        \item[(iii)] If $(i,l) \in \Psi''(j)$ and $i^\ast \in L(j)$, \eqref{eq_LBindiff_Cl1b} holds for all $\mathcal{I} \in  I(j)^{[i]}, \mathcal{L} \in  L(j)^{[l]}$ with $i^\ast \in \mathcal{L}$.
    \end{itemize}
    % If additionally $i^\ast \in \mathcal{I} \in I(j)^{[i]}$, 
    % \eqref{eq_LBindiff_Cl1b} holds even for all $\mathcal{L} \in L(j)^{[l']}$, where $l'\in \{0,\dots,|L(j)|\}$ fulfills $i+2l'\geq 2d_j$. Similarly, if additionally $i^\ast \in \mathcal{L} \in L(j)^{[l]}$, \eqref{eq_LBindiff_Cl1b} holds for all $\mathcal{I} \in I(j)^{[i']}$, where $i'\in \{0,\dots,|I(j)|\}$ fulfills $i'+2l \geq 2d_j -1$.\\[0.5em]
    %
    \textbf{Proof of Claim~1:} To prove (i), suppose  $(i,l) \in \Psi(j)$, that is, $i\in \{0,\dots,|I(j)|\}$ and $l\in \{0,\dots,|L(j)|\}$ are such that $i+2l \geq 2d_j+1$. Let $\mathcal{I} \in  I(j)^{[i]}$ and $ \mathcal{L} \in  L(j)^{[l]}$ be arbitrary but fixed. Define $\widetilde{P}_\succ$ as a modification of $P_\succ$  via  $\widetilde{P}_\succ^{(x,y)} \coloneqq P_\succ^{(x,y)}$ for any $(x,y) \in \mathcal{A} \times \mathcal{A} $ with $x\not=j\not=y$, 
    \begin{align}\label{eq:LB_proof_construction1}
        \widetilde{P}_{\succ}^{(j,z)} \coloneqq P_{\cong}^{(j,z)}, 
        \quad \widetilde{P}_{\cong}^{(j,z)} \coloneqq P_{\succ}^{(j,z)} 
        \quad \text{and} \quad \widetilde{P}_{\prec}^{(j,z)} \coloneqq P_{\prec}^{(j,z)}
    \end{align}
    for all $z\in \mathcal{I}$, and 
    \begin{align}\label{eq:LB_proof_construction2}
        \widetilde{P}_{\succ}^{(j,z)} \coloneqq  P_{\prec}^{(j,z)}, \quad 
        \widetilde{P}_{\cong}^{(j,z)} \coloneqq P_{\cong}^{(j,z)} 
        \quad \text{and} \quad \widetilde{P}_{\prec}^{(j,z)} \coloneqq P_{\succ}^{(j,z)}
    \end{align}
    for all $z\in \mathcal{L}$. By construction of $\widetilde{P}_\succ$, we have\footnote{In fact, the difference $\text{CP}(P_\succ,i^\ast) - \text{CP}(\widetilde{P}_\succ,i^\ast)$ is $1$ resp. $1/2$ resp. $0 $ if $i^\ast \in \mathcal{L}$ resp. $i^\ast \in \mathcal{I}$ resp. $i^\ast \not\in \mathcal{I} \cup \mathcal{L}$.} $\text{CP}(\widetilde{P}_\succ,i^\ast) \leq \text{CP}(P_\succ,i^\ast)$. Moreover, as the modes of the $(j,z)$-components of $P_\succ$ have been \textit{flipped to the ``category'' $j\succ z$} for $z\in \mathcal{I} \cup \mathcal{L}$ and remained unchanged otherwise, we have 
    \begin{equation*}
        \text{CP}(\widetilde{P}_{\succ}, j) = \text{CP}(P_{\succ},j) + \frac{1}{2}|\mathcal{I}| + |\mathcal{L}| = \text{CP}(P_{\succ},j) + \frac{i}{2} + l.
    \end{equation*}
    As $i+2l \geq 2d_j +1$ holds by assumption, we thus get
    \begin{align*}
    \text{CP}(\widetilde{P}_{\succ}, j) &\geq
    \text{CP}(P_{\succ},j) + d_j + 1/2 \\
    &= \text{CP}(P_{\succ},j) + (\text{CP}(P_{\succ},i^\ast) - \text{CP}(P_{\succ},j)) + 1/2 \\
    &= \text{CP}(P_{\succ},i^\ast) + 1/2 
    \geq \text{CP}(\widetilde{P}_{\succ},i^\ast) + 1/2.
    \end{align*}
    This shows\footnote{In fact, by construction we even have $\mathcal{C}(\widetilde{P}_{\succ}) = \{j\}$.} $i^\ast \not\in \mathcal{C}(P_\succ)$, and by assumption on $\Algo{A}$, the event $\mathcal{E} \coloneqq \{i^\ast \in \mathcal{C}(\mathbf{P})\} \in \mathcal{F}_{\tau^{\Algo{A}}}$ has the properties
    \begin{align*}
        \prob_{P_{\succ}}(\mathcal{E}) \geq 1-\delta, \quad \prob_{\widetilde{P}_{\succ}}(\mathcal{E}) \leq \delta.
    \end{align*}
    Consequently, Lemma~\ref{Le_from_Kaufmann} and part~(ii) of Lemma~\ref{Le_Le1_for_LB_Proof} assure
    \begin{equation}\label{eq:proof_LB_claim1_eq}
        \sum\nolimits_{x<y} \exptd_{\mathbf{P}} \left[ \tau_{x,y}^{\Algo{A}} \right] \KL\left(P_{x,y}, \widetilde{P}_{x,y} \right) \geq \mathrm{kl}\left(\prob_{P_\succ}(\mathcal{E}),\prob_{\widetilde{P}_\succ}(\mathcal{E})\right) \geq \mathrm{kl}(1-\delta,\delta) \geq \ln \frac{1}{2.4\delta}.
    \end{equation}
    In case $x\not= j \not= y$, $\widetilde{P}_{x,y} = P_{x,y}$ assures $\KL\left(P_{x,y}, \widetilde{P}_{x,y} \right) = 0$. In case $z\in \mathcal{I} \subseteq I(j)$, we have 
    \begin{equation*}
        \KL\left(P_{j,z}, \widetilde{P}_{j,z} \right) = \KL\left(
        \left( P_{\succ}^{(j,z)},P_{\cong}^{(j,z)},P_{\prec}^{(j,z)}\right), \left( P_{\cong}^{(j,z)},P_{\succ}^{(j,z)},P_{\prec}^{(j,z)} \right) \right) \leq D_{j,z}(\mathbf{P}),
    \end{equation*}
    and in case $z\in \mathcal{L} \subseteq L(j)$ we similarly see
    \begin{equation*}
        \KL\left(P_{j,z}, \widetilde{P}_{j,z} \right) = \KL\left(
        (P_{\succ}^{(j,z)},P_{\cong}^{(j,z)},P_{\prec}^{(j,z)}), (P_{\prec}^{(j,z)},P_{\cong}^{(j,z)},P_{\succ}^{(j,z)}) \right) \leq D_{j,z}(\mathbf{P}).
    \end{equation*}
    Combining these estimates with \eqref{eq:proof_LB_claim1_eq} proves \eqref{eq_LBindiff_Cl1b}.\\[0.5em]
    To prove (ii), suppose now $(i,l) \in \Psi'(j)$ and $\mathcal{I} \in I(j)^{[i]}$, $\mathcal{L} \in L(j)^{[l]}$ with $i^\ast \in \mathcal{I}$ are given. Similarly as above, one may construct an instance $\mathbf{\widetilde{P}}$, which differs from $\mathbf{P}$ only on positions $(j,z)$, $z\in \mathcal{I} \cup \mathcal{L} \cup \{i^\ast\}$, such that \eqref{eq:LB_proof_construction1} for all $z\in \mathcal{I}$ and \eqref{eq:LB_proof_construction1} for all $z\in \mathcal{L}$ and $\widetilde{P}_{\succ}^{(j,i^\ast)} = P_{\cong}^{(j,i^\ast)}$, $\widetilde{P}_{\cong}^{(j,i^\ast)} = P_{\succ}^{(j,i^\ast)}$ and $\widetilde{P}_{\prec}^{(j,i^\ast)} = P_{\prec}^{(j,i^\ast)}$.
    Then, $\text{CP}(\widetilde{P}_{\succ}, i^\ast) = \text{CP}(P_{\succ},i^\ast)$ holds and again $\text{CP}(\widetilde{P}_{\succ},j) = \text{CP}(P_{\succ},j) + \frac{i}{2} + l$. Due to $i+2l \geq 2d_j -1$ we obtain $\text{CP}(\widetilde{P}_{\succ},j)  \geq \text{CP}(\widetilde{P}_{\succ},j) + 1/2 $ and thus $i^\ast \not\in \mathcal{C}(\widetilde{P}_\succ)$.
    Thus, the same argumentation as above shows that \eqref{eq_LBindiff_Cl1b} also holds in this case.\\[0.5em]
    For proving (iii), construct $\widetilde{P}_{\succ}$ such that it differs from $P$ only on positions $(j,z)$, $z\in \mathcal{I} \cup \mathcal{L} \cup \{i^\ast\}$, fulfills \eqref{eq:LB_proof_construction1} for all $z\in \mathcal{I}$ and \eqref{eq:LB_proof_construction1} for all $z\in \mathcal{L}$ and further  $\widetilde{P}_{\succ}^{(j,i^\ast)} = P_{\prec}^{(j,i^\ast)}$, $\widetilde{P}_{\cong}^{(j,i^\ast)} = P_{\cong}^{(j,i^\ast)}$ and $\widetilde{P}_{\prec}^{(j,i^\ast)} = P_{\succ}^{(j,i^\ast)}$. Then, $\text{CP}(\widetilde{P}_{\succ},i^\ast) = \text{CP}(P_{\succ},i^\ast) - 1$ holds, and the assumptions stated in (iii) suffice to show $i^\ast \not\in \mathcal{C}(\widetilde{P}_{\succ})$. Therefore, repeating the arguments from above shows \eqref{eq_LBindiff_Cl1b}.$\blacksquare$\\[0.5em]
    % \begin{equation*}
    %     \KL\left(P_{j,z}, \widetilde{P}_{j,z} \right) \leq  \frac{(P_{\succ}^{(j,z)} - \widetilde{P}_{\cong}^{(j,z)})^{2}}{\widetilde{P}_{\cong}^{(j,z)}} + \frac{(P_{\cong}^{(j,z)} - \widetilde{P}_{\succ}^{(j,z)})^{2}}{\widetilde{P}_{\succ}^{(j,z)}} 
    % \end{equation*}   
    % \color{blue}
    % Now, ``flip all entries $(j,z)$ for $z\in \mathcal{I}$ from $1/2$ to $1$ and all entries $(j,z)$ for $z\in \mathcal{L}$ from $0$ to $1$''. \textbf{More formally! Complete the proof of Claim 1!} Note that for $(i,l) \in \Psi'(j)$ we can flip $(j,i^\ast)$ to $1$, which will (a) increase the partial Copeland score of $j$ by $1/2$ and (b) decrease the Copeland score of $i^\ast$ by $1/2$, whence we only have to further increase the Copeland score of $j$ by $2d_j-1$ in order to obtain $i^\ast \not= CP$.\hfill $\blacksquare$\\[0.5em] 
    % \color{black}
    % For any $\mathcal{I} \in I(j)^{[i]}$, $\mathcal{L} \in L(j)^{[l]}$, we obtain \textit{as usual} \textbf{(more precisely!)} the estimate 
    % \begin{equation*}
    %     \ln \frac{1}{2.4\delta} \leq \sum\nolimits_{z\in \mathcal{I}} \exptd_{P_\succ}\left[\tau_{jz}^{\Algo{A}}\right] \Delta^{2}(j,z) + \sum\nolimits_{z\in \mathcal{L}} \exptd_{P_\succ}\left[\tau_{jz}^{\Algo{A}}\right] \Delta^{2}(j,z).        
    % \end{equation*}
    As $\mathbf{P}$ is fixed, we may simply write $D_{x,y}$ for $D_{x,y}(\mathbf{P})$ throughout the rest of the proof.
    First, let $(i,l) \in \Psi(j)$ be arbitrary but fixed, i.e.,  $i\in \{0,\dots,|I(j)|\}$ and $l\in \{0,\dots,|L(j)|\}$  and  $i+2l \geq 2d_j + 1$ hold. According to Part~(i) of Claim~1,  \eqref{eq_LBindiff_Cl1b} holds for any of the $\binom{|I(j)|}{i} \binom{|L(j)|}{l}$ many $(\mathcal{I}, \mathcal{L}) \in I(j)^{[i]} \times L(j)^{[l]}$. When summing \eqref{eq_LBindiff_Cl1b} over all such $(\mathcal{I},\mathcal{L})$,
    % When summing \eqref{eq_LBindiff_Cl1} over all $\binom{|I(j)|}{i} \binom{|L(j)|}{l}$ many  $(\mathcal{I}, \mathcal{L}) \in I(j)^{[i]} \times \mathcal{L}^{[l]}$, 
    the summand $\exptd_{\mathbf{P}}\left[\tau_{jz}^{\Algo{A}}\right] D_{j,z}$ appears exactly $\binom{|I(j)|-1}{i-1} \binom{|L(j)|}{l} \IND{i\geq 1}$ many times if $z\in I(j)$, and it appears $\binom{|I(j)|}{i} \binom{|L(j)|-1}{l-1} \IND{l\geq 1}$ many times  if $z \in L(j)$. Consequently, we have 
    \begin{align*}
        \binom{|I(j)|}{i} &\binom{|L(j)|}{l} \ln\frac{1}{2.4\delta} \\
        &\leq  \binom{|I(j)|-1}{i-1} \binom{|L(j)|}{l} \IND{i\geq 1 } \sum\nolimits_{z\in \mathcal{I}} \exptd_{\mathbf{P}}\left[\tau_{jz}^{\Algo{A}}\right] D_{j,z} \\
        &\phantom{=}+ \binom{|I(j)|}{i} \binom{|L(j)|-1}{l-1} \IND{l\geq 1 } \sum\nolimits_{z\in \mathcal{L}} \exptd_{\mathbf{P}}\left[\tau_{jz}^{\Algo{A}}\right] D_{j,z} \\
        &\leq \max\nolimits_{z\in I(j)\cup L(j)} D_{j,z} \cdot \sum\nolimits_{z\in I(j) \cup L(j)} \exptd_{\mathbf{P}}\left[ \tau_{jz}^{\Algo{A}} \right] \\
        &\phantom{=} \cdot \left[  \binom{|I(j)|-1}{i-1} \binom{|L(j)|}{l} \IND{i\geq 1 } + \binom{|I(j)|}{i} \binom{|L(j)|-1}{l-1} \IND{l\geq 1 } \right].
    \end{align*}
    Thus, we obtain for $\tau_{j}^{\Algo{A}} = \sum_{z\not= j} \tau_{jz}^{\Algo{A}}$ the estimate
    % Recalling the definition of 
    % \begin{equation*}
    %     R(I,L,i,l) = \frac{\binom{I}{i} \binom{L}{l}}{\binom{I-1}{i-1} \binom{L}{l} \bold{1}_{\{i\geq 1\}} + \binom{I}{i} \binom{L-1}{l-1} \bold{1}_{\{l\geq 1\}}}
    % \end{equation*}
    % we thus obtain for   the estimate   
    \begin{align}
        \exptd_{\mathbf{P}}&\left[\tau_j^{\Algo{A}} \right] \geq \sum\nolimits_{z\in I(j) \cup L(j)} \exptd_{\mathbf{P}}\left[ \tau_{jz}^{\Algo{A}} \right] \notag \\
        &\geq \frac{\binom{|I(j)|}{i} \binom{|L(j)|}{l}}{\binom{|I(j)|-1}{i-1} \binom{|L(j)|}{l} \IND{i\geq 1 } + \binom{|I(j)|}{i} \binom{|L(j)|-1}{l-1} \IND{l\geq 1 }} \left( \ln \frac{1}{2.4\delta} \right) \min\nolimits_{z\in I(j) \cup L(j)} \frac{1}{D_{j,z}}.  \label{eq:LB_p1}
    \end{align}
    % Next, suppose as in (ii) that $i^\ast I(j)$ and that $i\leq |I(j)|$ holds. According to (ii) in Claim~1, \eqref{eq_LBindiff_Cl1b} is fulfilled for all any of the $\binom{|I(j)|-1}{i-1}\binom{|L(j)|}{l}$ many $(\mathcal{I},\mathcal{L}) \in I(j)^{[i]} \times L(j)^{[l]}$ with $i^\ast \in \mathcal{I}$. When summing over all these $(\mathcal{I},\mathcal{L})$
    Next, suppose $i^\ast \in I(j)$ and $(i,l) \in \Psi'(j)$ be fixed for the moment. For any of the $\binom{|I(j)|-1}{i} \binom{|L(j)|}{l}$ many $(\mathcal{I}, \mathcal{L}) \in I(j)^{[i]} \times \mathcal{L}(j)^{[l]}$ with $i^\ast \in \mathcal{I}$, Part~(ii) of Claim~1 yields that \eqref{eq_LBindiff_Cl1b} holds. Summing this over all such $(\mathcal{I},\mathcal{L})$, we observe:
    \begin{itemize}
        \item[$\bullet$] The summand $\exptd_{\mathbf{P}}\left[\tau_{ji^\ast}^{\Algo{A}}\right] D_{j,i^\ast}$ appears $\binom{|I(j)|-1}{i} \binom{|L(j)|}{l}$ many times.
        \item[$\bullet$] For $z\in I(j) \setminus \{i^\ast\}$, the summand $\exptd_{\mathbf{P}}\left[\tau_{jz}^{\Algo{A}}\right] D_{j,z}$ appears $\binom{|I(j)|-2}{i-1} \binom{|L(j)|}{l}$  many times if $|I(j)|>i\geq 1$, and it does not appear at all if $i=0$. Thus, this summand appears $\binom{|I(j)|-2}{i-1} \binom{|L(j)|}{l} \IND{|I(j)|>i\geq 1 }$ many times.
        \item[$\bullet$] For $z\in L(j)$, the summand $\exptd_{\mathbf{P}}\left[\tau_{jz}^{\Algo{A}}\right] D_{j,z}$ appears $\binom{|I(j)|-1}{i} \binom{|L(j)|-1}{l-1} \IND{l\geq 1 }$ many times.
    \end{itemize}
    Thus, we obtain 
    \begin{align*}
        &\binom{|I(j)|-1}{i} \binom{|L(j)|}{l} \ln\frac{1}{2.4\delta} \\
        &\leq  \binom{|I(j)|-1}{i} \binom{|L(j)|}{l} \exptd_{\mathbf{P}}\left[\tau_{ji^\ast}^{\Algo{A}}\right] D_{j,i^\ast} \\
        &\phantom{=} +\binom{|I(j)|-2}{i-1} \binom{|L(j)|}{l} \IND{i\geq 1 } \sum\nolimits_{z\in \mathcal{I}} \exptd_{\mathbf{P}}\left[\tau_{jz}^{\Algo{A}}\right] D_{j,z}\\
        &\phantom{=} + \binom{|I(j)|-1}{i} \binom{|L(j)|-1}{l-1} \IND{l\geq 1 } \sum\nolimits_{z\in \mathcal{L}} \exptd_{\mathbf{P}}\left[\tau_{jz}^{\Algo{A}}\right] D_{j,z} \\
        &\leq \max\nolimits_{z\in I(j)\cup L(j)} D_{j,z} \cdot \sum\nolimits_{z\in I(j) \cup L(j)} \exptd_{\mathbf{P}}\left[ \tau_{jz}^{\Algo{A}} \right] \\
        &\phantom{=} \cdot \left[ \binom{|I(j)|-1}{i} \binom{|L(j)|}{l} +   \binom{|I(j)|-2}{i-1} \binom{|L(j)|}{l} \IND{i\geq 1 } +  \binom{|I(j)|-1}{i} \binom{|L(j)|-1}{l-1} \IND{l\geq 1 } \right],
    \end{align*}
    which shows similarly as above that 
    \begin{align}
        \exptd_{\mathbf{P}}\left[ \tau_j^{\Algo{A}}\right] \geq 
        &\frac{\binom{|I(j)|-1}{i} \binom{|L(j)|}{l}}{\binom{|I(j)|-1}{i} \binom{|L(j)|}{l} +   \binom{|I(j)|-2}{i-1} \binom{|L(j)|}{l} \IND{i\geq 1 } +  \binom{|I(j)|-1}{i} \binom{|L(j)|-1}{l-1} \IND{l\geq 1 }} \left( \ln \frac{1}{2.4\delta} \right) \notag \\
        &\cdot \min\nolimits_{z\in I(j) \cup L(j)} \frac{1}{D_{j,z}}. \label{eq:LB_p2}
    \end{align}
    Finally, suppose $i^\ast \in L(j)$ and let $(i,l) \in \Psi''(j)$ be fixed for the moment. For any of the $\binom{|I(j)|}{i}\binom{|L(j)|-1}{l}$ many $(\mathcal{I},\mathcal{L}) \in I(j)^{[i]} \times L(j)^{[l]}$ with $i^\ast \in L(j)$, Part~(iii) of Claim~1 yields that \eqref{eq_LBindiff_Cl1b} holds. When summing over all these $(\mathcal{I},\mathcal{L})$, we see:
    \begin{itemize}
        \item[$\bullet$] The summand $\exptd_{\mathbf{P}}\left[\tau_{ji^\ast}^{\Algo{A}}\right] D_{j,i^\ast}$ appears $\binom{|I(j)|}{i} \binom{|L(j)|-1}{l}$ many times.
        \item[$\bullet$] For $z\in I(j)$, the summand $\exptd_{\mathbf{P}}\left[\tau_{jz}^{\Algo{A}}\right] D_{j,z}$ appears $\binom{|I(j)|-1}{i-1} \binom{|L(j)|-1}{l} \IND{i\geq 1 }$ many times.
        \item[$\bullet$] For $z\in L(j) \setminus \{i^\ast\}$, the summand $\exptd_{\mathbf{P}}\left[\tau_{jz}^{\Algo{A}}\right] D_{j,z}$ appears $\binom{|I(j)|}{i} \binom{|L(j)|-2}{l-1} \IND{l\geq 1 }$ many times.
    \end{itemize}
    Analogously as above, we obtain 
    \begin{align}
        \exptd_{\mathbf{P}}\left[ \tau_j^{\Algo{A}}\right] \geq 
        &\frac{\binom{|I(j)|}{i} \binom{|L(j)|-1}{l}}{\binom{|I(j)|}{i} \binom{|L(j)|-1}{l} +   \binom{|I(j)|-1}{i-1} \binom{|L(j)|-1}{l} \IND{i\geq 1 } +  \binom{|I(j)|}{i} \binom{|L(j)|-2}{l-1} \IND{l\geq 1 }} \left( \ln \frac{1}{2.4\delta} \right) \notag \\
        &\cdot \min\nolimits_{z\in I(j) \cup L(j)} \frac{1}{D_{j,z}}. \label{eq:LB_p3}
    \end{align}
    As \eqref{eq:LB_p1} holds for all $(i,l) \in \Psi(j)$, \eqref{eq:LB_p2} for all $(i,l) \in \Psi'(j)$, if $i^\ast \in I(j)$, and \eqref{eq:LB_p3} holds for all $(i,l) \in \Psi''(j)$ if $i^\ast \in L(j)$, combining these estimates concludes the proof.
    \end{proof}

    For the sake of comparison, let us state the consequences from Theorem~\ref{theorem:lower_bound_detailed} for the particular case when the indifferences are non-dominant in the sense that $I(z)=\emptyset$ for all $z$. Note that the maximum appearing therein is exactly the same term as that in the lower bound for Copeland winner identification without indifferences (Thm.~\ref{theorem:lower_bound_wo_ind_detailed}).
    \begin{corollary}\label{corollary:lower_bound_wo_indifferences}
        If $\Algo{A}$ is an algorithm that correctly identifies the Copeland winner with confidence $1-\delta$ for any $\mathbf{P}$ (possibly with indifferences), then we have for all such $\mathbf{P}$ with  $\min_{i<j} \min\{P_{\succ}^{(i,j)},P_{\cong}^{(i,j)},P_{\prec}^{(i,j)}\}>0$,  $\max_{z\in \mathcal{A}} |I(z)| = 0$ and $\mathcal{C}(\mathbf{P}) = \{i^\ast\}$ that
        \begin{align*}
            \exptd_{\mathbf{P}}\left[\tau^{\Algo{A}} \right] \geq 
            \ln \frac{1}{2.4\delta} \sum_{j\not=i^\ast} &\max \left\{ \frac{|L(j)|}{d_j +1 } \IND{|L(j)| \geq d_j + 1  } , \frac{|L(j)|-1}{|L(j)| + d_j - 2} \IND{i^\ast \in L(j) } \right\}  \\ &\cdot \ \min_{z\in L(j) \cup I(j)} \frac{1}{D_{j,z}(\mathbf{P})}.
        \end{align*}        
    \end{corollary}    
    \begin{proof}[Proof of Corollary~\ref{corollary:lower_bound_wo_indifferences}]
        Suppose $\mathbf{P}$ is such that $\max_{z\in \mathcal{A}} |I(z)| = 0$ and recall the definitions of $C_j$, $C''_j$, $\Psi(j)$ and $\Psi''(j)$ from Thm.~\ref{theorem:lower_bound_detailed}. Then, $\text{CP}(j)$ and $d_j$ are integers for any $j\in \mathcal{A}$. Moreover, for any $j\in \mathcal{A} \setminus \{i^\ast\}$,  $|I(j)|=0$ directly implies $\Psi(j)=\{(0,d_j+1),\dots,(0,|L(j)|)\}$ and $\Psi''(j) = \{(0,d_j -1 ),\dots,(0,|L(j)|-1)\}$. Note that $\Psi(j) \not= \emptyset$ iff $|L(j)| \geq  d_j +1 $, whereas  
         $|L(j)| \geq d_j$ shows that $\Psi''(j) \not= \emptyset$ in any case. Using that   $\binom{a}{b} / \binom{a-1}{b-1} = \frac{a}{b}$ holds for any $a,b\in \N$ with $a\leq b$, we thus obtain
        \begin{align*}
            C_j &= \max\nolimits_{(i,l) \in \Psi(j)} \frac{\binom{|I(j)|}{i} \binom{|L(j)|}{l}}{\binom{|I(j)|}{i} \binom{|L(j)|-1}{l-1}} \\
            &= \max\nolimits_{l \in \{d_j+1,\dots, |L(j)|\}} \frac{|L(j)|}{l} 
            = \frac{|L(j)|}{d_j +1 } \IND{|L(j)| \geq d_j +1 }
        \end{align*}
        and similarly 
        \begin{align*}
            C''_j &= \max\nolimits_{(i,l) \in \Psi''(j)} \frac{\binom{|I(j)|}{i} \binom{|L(j)|-1}{l}}{\binom{|I(j)|}{i} \binom{|L(j)|-1}{l} + \binom{|I(j)|}{i} \binom{|L(j)|-2}{l-1}} \\
            % &= \max_{l\in \{d_j,\dots,|L(j)|-1\}} \frac{\binom{|L(j)|-1}{l}}{\binom{|L(j)|-1}{l} + \binom{|L(j)|-2}{l-1}} \\
            &= \max\nolimits_{l\in \{d_j -1 ,\dots,|L(j)|-1\}} \frac{\binom{|L(j)|-1}{l}}{\binom{|L(j)|-1}{l} + \frac{l}{|L(j)|-1} \binom{|L(j)|-1}{l}} \\
            % &= \max_{l\in \{d_j,\dots,|L(j)|-1\}} \frac{1}{1 + \frac{l}{|L(j)|-1}} \\
            &= \max\nolimits_{l\in \{d_j -1 ,\dots,|L(j)|-1\}} \frac{|L(j)|-1}{|L(j)|-1+l} \\
            &= \frac{|L(j)|-1}{|L(j)| + d_j - 2 }.
        \end{align*}
        Thus, the statement follows from Thm.~\ref{theorem:lower_bound_detailed}.
    \end{proof}

    To conclude this section, we state in the following corollary worst-case consequences of Thm.~\ref{theorem:lower_bound_detailed} and Thm.~\ref{theorem:lower_bound_wo_ind_detailed}. They show in particular that -- in both learning scenarios with and without indifferences -- identifying the Copeland winner of $\mathbf{P}$ requires $\Omega(n^{2})$ samples in the worst case.
    \begin{corollary}
        Let $f:\N \rightarrow \N$ with $1\leq f(n) \leq \frac{n}{2}-1$ for all $n\in \N$ and $f(n) \in o(n)$ as $n\rightarrow \infty$ and let $\Delta\in (0,1/6)$ be arbitrary.
        \begin{itemize}
            \item[(i)]
            There exists a sequence $(\mathbf{P}^{n})_{n\in \N}$ of instances without indifferences with $(P^{n})^{(i,j)}_{\succ} \in \{1/2\pm \Delta\}$ for all $i,j\in \mathcal{A}$ and $\mathrm{CP}^\ast(\mathbf{P}^{n}) \geq  \lceil \frac{n}{2} + f(n)\rceil$ such that
            \begin{equation*}
                \exptd_{\mathbf{P}^{n}}\left[ \tau^{\Algo{A}} \right] \in \Omega\left( \frac{n^2}{f(n)\Delta^2} \ln \frac{1}{\delta} \right) 
            \end{equation*}
            for any algorithm $\Algo{A}$ that correctly identifies the Copeland winner of any $\mathbf{P}$ without indifferences with confidence $1-\delta$.
            \item[(ii)]
            There exists a sequence $(\mathbf{P}^{n})_{n\in \N}$ of instances with $(P^{n})^{(i,j)}_{\succ}, (P^{n})^{(i,j)}_{\cong}, (P^{n})^{(i,j)}_{\prec} \in \{1/3-\Delta, 1/3+2\Delta\}$ for all $i,j\in \mathcal{A}$ and $\mathrm{CP}^\ast(\mathbf{P}^{n}) \geq \lceil \frac{n}{2} + f(n)\rceil$ such that
            \begin{equation*}
                \exptd_{\mathbf{P}^{n}}\left[ \tau^{\Algo{A}} \right] \in \Omega\left( \frac{n^2}{f(n)\Delta^2} \ln \frac{1}{\delta} \right) 
            \end{equation*}
            for any algorithm $\Algo{A}$ that correctly identifies the Copeland winner of any $\mathbf{P}$ with indifferences with confidence $1-\delta$.
        \end{itemize}
    \end{corollary}
    \begin{proof}
        Let $f$ and $\Delta$ be as stated above. We start with the proof of (i). By assumption on $f$ there exists $n_0 \in \N$ with $n-1-\lceil \frac{n}{2} + f(n) \rceil \geq 1$ and $\lfloor \frac{n-1}{2} \rfloor \geq f(n) +2$ for all $n\geq n_0$. For $n < n_0$, let $\mathbf{P}^{n}$ be an arbitrary allowed instance. For arbitrary but fixed $n\geq n_0$, fix a set $L_n \subseteq \mathcal{A}$ of size $|L_n| = n -1 - \lceil \frac{n}{2} + f(n) \rceil $ and define $\mathbf{P}^{n} = ((P^{n})_{\succ}^{(x,y)})_{x,y}$ via 
        \begin{equation*}
            (P^{n})_{\succ}^{(1,y)} \coloneqq \begin{cases} 1/2 + \Delta, \quad &\text{if } y \in \mathcal{A} \setminus L_n,\\
            1/2 - \Delta, \quad &\text{if } y \in L_n\end{cases}
        \end{equation*}
        for $2\leq y\leq n$ and
        \begin{equation*}
            (P^{n})_{\succ}^{(x,y)} \coloneqq \begin{cases} 1/2 + \Delta, \quad &\text{if } x+y \text{ is even},\\
            1/2 - \Delta,\quad &\text{if } x+y \text{ is odd}\end{cases}
        \end{equation*}
        for $2\leq x<y\leq n$. Then, $\mathbf{P}^{n}$ has no indifferences. Moreover, $\text{CP}(\mathbf{P}^{n},1) = n-1-|L_n| = \lceil \frac{n}{2} + f(n)\rceil$ and 
         $\text{CP}(\mathbf{P}^{n},j) \geq \lfloor \frac{n-1}{2} \rfloor - 1 $ for $j\in \mathcal{A} \setminus \{1\}$ hold, which shows $\mathcal{C}(\mathbf{P}^{n}) = \{1\}$ and $d_j(\mathbf{P}^{n}) = \text{CP}(\mathbf{P}^{n},1)-\text{CP}(\mathbf{P}^{n},j) \leq  f(n) +1$. By construction, $\lfloor \frac{n-1}{2} \rfloor \leq |L(\mathbf{P}^{n},j)| \leq \lceil \frac{n-1}{2} \rceil +1$ is fulfilled, and thus by choice of $n_0$ also $|L(\mathbf{P}^{n},j)| \geq \left\lfloor \frac{n-1}{2} \right\rfloor \geq f(n)+2 = d_j(\mathbf{P}^{n}) + 1$ holds.
         % In particular, $f(n) \leq 2n-3 \leq 2\lfloor (n-1)/2\rfloor -2 $ assures 
         % \begin{align*}
         %    |L(\mathbf{P}^{n},j)| \geq \left\lfloor \frac{n-1}{2} \right\rfloor \geq
         %     f(n) - \left\lfloor \frac{n-1}{2} \right\rfloor + 2 =  d_j(\mathbf{P}^{n}) + 1.
         % \end{align*}
         Using that Lemma~\ref{lemma:lower_bound_1} and $\Delta < 1/6$ imply
         \begin{align*}
            \kappa_{x,y}(\mathbf{P}) = \mathrm{kl}(P_{\succ}^{(x,y)},P_{\prec}^{(x,y)}) \leq \frac{(P_{\succ}^{(x,y)} - P_{\prec}^{(x,y)})^2}{P_{\prec}^{(x,y)}(1-P_{\prec}^{(x,y)})} = \frac{4\Delta^2}{(1/2-\Delta)(1/2+\Delta)} \leq 16\Delta^2,
         \end{align*}
         Thm.~\ref{theorem:lower_bound_wo_ind_detailed} yields
         \begin{align*}
             \exptd_{\mathbf{P}^{n}}\left[ \tau^{\Algo{A}}\right] &\geq \ln \frac{1}{2.4\delta} \sum\nolimits_{j\not= i^\ast} \frac{|L(\mathbf{P}^{n},j)|}{d_{j}(\mathbf{P}^{n})} \min_{z\in L(j)} \frac{1}{\kappa_{j,z}(\mathbf{P}^{n})} \\
             &\geq \frac{1}{16\Delta^{2}} \left(\ln \frac{1}{2.4\delta}\right) \sum_{j\not=i^\ast} \frac{\lfloor (n-1)/2 \rfloor}{f(n) +1}\\
             &\geq \frac{1}{32\Delta^{2}} \left(\ln \frac{1}{2.4\delta}\right) \sum_{j\not=i^\ast} \frac{n-1}{f(n)+1}\\
             &\geq \frac{1}{32\Delta^{2}} \left(\ln \frac{1}{2.4\delta}\right) \frac{(n-1)^2}{f(n)+1},
         \end{align*}
        which concludes the proof of (i).\\[0.5em]
        To prove (ii), define $n_0$ as before and fix allowed arbitrary $\mathbf{P}^{n}$ for $n< n_0$. For arbitrary but fixed $n\geq n_0$, fix again $L_n \subseteq \mathcal{A}$ of size $n-1-\lceil \frac{n}{2} + f(n) \rceil$ and define the instances $\mathbf{P}^{n} = ((P^{n})_{\succ}^{(x,y)})_{x,y}$ via 
        \begin{align*}
            (P^{n})_{\succ}^{(1,y)} &\coloneqq \begin{cases} 
            1/3 + 2\Delta, \quad &\text{if } y\in \mathcal{A} \setminus L_n,\\
            1/3 - \Delta, \quad &\text{if } y\in L_{n},
            \end{cases}\\
            (P^{n})_{\succ}^{(y,1)} &\coloneqq \begin{cases} 
            1/3 - \Delta, \quad &\text{if } y\in \mathcal{A} \setminus L_n,\\
            1/3 + 2\Delta, \quad &\text{if } y\in L_{n},
            \end{cases}
        \end{align*}
        for $2\leq y\leq n$ and 
        \begin{equation*}
            (P^{n})_{\succ}^{(x,y)} \coloneqq \begin{cases} 
            1/3 + 2\Delta, \quad &\text{if ($x+y$ is even and $x<y$) or ($x+y$ is odd and $x>y$)}\\
            1/3 - \Delta, \quad &\text{if ($x+y$ is odd and $x<y$) or ($x+y$ is even and $x>y$)}\\
            \end{cases}
        \end{equation*}
        for distinct $x,y\in \{2,\dots,n\}$.
        Then, $\mathbf{P}^{n}$ has indifferences and fulfills  $(P^{n})_{\cong}^{(x,y)} = 1/3-\Delta$ for any distinct $x,y\in \mathcal{A}$, which directly implies $I(j) = \emptyset$ for any $j\in \mathcal{A}$. 
        By construction, we see similarly as above $\text{CP}(\mathbf{P}^{n},1) = \lceil \frac{n}{2} + f(n)\rceil$, $\mathcal{C}(\mathbf{P}^{n}) = \{1\}$ and $d_{j}(\mathbf{P}^{n}) \leq f(n)+1$ 
        as well as $\lfloor \frac{n-1}{2} \rfloor \leq |L(\mathbf{P}^{n},j)|\leq \lceil \frac{n-1}{2} \rceil +1$ and $|L(\mathbf{P}^{n},j)| \geq  d_j(\mathbf{P}^{n}) + 1$ for $j\in \mathcal{A} \setminus \{1\}$. Regarding the construction of $\mathbf{P}^{n}$, Lemma \ref{lemma:lower_bound_1} implies due to $0<\Delta < 1/6$ the estimate
        \begin{align*}
            D_{x,y}(\mathbf{P}^{n}) &= \KL\left( \left( \frac{1}{3} + 2\Delta, \frac{1}{3}-\Delta, \frac{1}{3}-\Delta\right),  \left( \frac{1}{3}-\Delta , \frac{1}{3} + 2\Delta, \frac{1}{3}-\Delta\right) \right) \\
            &\leq \frac{(3\Delta)^{2}}{1/3 - \Delta} + \frac{(3\Delta)^2}{1/3 + 2\Delta} \leq 3^2(6 + 3)\Delta^2 = 81\Delta^2.
        \end{align*}
        With the use of Cor.~\ref{corollary:lower_bound_wo_indifferences} instead of Thm.~\ref{eq:thm_lower_bound_wo_ind_detailed}, following the same argumentation as in the proof of (i) thus lets us conclude 
        \begin{equation*}
            \exptd_{\mathbf{P}^{n}}\left[ \tau^{\Algo{A}}\right] 
            \geq \frac{1}{162\Delta^{2}} \left(\ln \frac{1}{2.4\delta}\right) \frac{(n-1)^2}{f(n)+1}.
        \end{equation*}
    \end{proof}

\clearpage

\section{Proof of Theorem \ref{theorem_pocowista}} \label{sec:proof_theorem_pocowista}

\textbf{Theorem 3.1.} \emph{	%	
	%	Assume that there exists no pair $i,j \in \mathcal{A}$ with $i\neq j$ such that $P_{\succ}^{(i,j)} = 1/3$ and $P_{\cong}^{(i,j)} = 1/3.$ 
	%	
	Let $\Algo{A}:= \Algo{POCOWISTA}.$
	For any dueling bandits problem with indifferences characterized by $\mathbf{P} = ((P_{\succ}^{(i,j)},P_{\cong}^{(i,j)}, P_{\prec}^{(i,j)}))_{i<j},$    such that there exists no pair $i,j \in \mathcal{A}$ with $i\neq j$ and $P_{\succ}^{(i,j)} = P_{\succ}^{(j,i)}= 1/3,$  it holds that
	\begin{align*}
		\prob\big(\hat{i}_{\Algo{A}} \in \mathcal{C}(\mathbf{P}) \mbox{ and } \tau^{\Algo{A}}(\mathbf{P}) \leq t(\mathbf{P},\delta) \big) 
		\geq 1- \delta,		
	\end{align*}
	where 
	\begin{align*}
		t(\mathbf{P},\delta) 
		= \sum_{i<j} t_0\big(  (P_{\succ}^{(i,j)}, P_{\cong}^{(i,j)},P_{\prec}^{(i,j)}), \nicefrac{\delta}{\binom{n}{2}} \big), \\
		t_0\big(  (p_1,p_2,p_3), \delta \big) 
		= \frac{c_1 p_{(1)}	\ln\left(	\sqrt{\frac{2\, c_2}{\delta}   } \frac{p_{(1)}	}{(p_{(1)}- p_{(2)})} 		\right)}{(p_{(1)}- p_{(2)})^2}, 
	\end{align*}
	$p_{(1)}\geq p_{(2)} \geq p_{(3)}$ is the order statistic of $p_1,p_2,p_3 \in [0,1]$ with $\sum_{i=1}^3 p_i =1$ and $c_1 = 194.07$ and $c_2 = 79.86.$
	}
\begin{proof}
Let $P$ be a categorical distribution with categories $c_1,c_2$ and $c_3$ having probabilities $p_1,p_2$ and $p_3,$ i.e., $P(c_i) =p_i$ for $i=1,2,3,$ such that $p_{(1)}>p_{(2)}\geq p_{(3)}.$
Theorem 9 in \cite{jain2022pac} states that running \Algo{PPR-1v1} with $\tilde \delta \in[0,1]$ as the desired error probability for identifying the mode of $P,$ leads to a sample complexity of at most
  $$t_0\big(  (p_1,p_2,p_3), \tilde \delta \big) 
  = \frac{c_1 p_{(1)}	\ln\left(	\sqrt{\frac{2\, c_2}{\tilde \delta}   } \frac{p_{(1)}	}{(p_{(1)}- p_{(2)})} 		\right)}{(p_{(1)}- p_{(2)})^2},$$
for identifying the mode with probability at least $1-\tilde \delta.$

In the worst case, \Algo{POCOWISTA} has to use \Algo{PPR-1v1} with an error probability of $\tilde \delta = \nicefrac{\delta}{\binom{n}{2}}$ for each ternary distribution $\tern,$ where $i<j.$
Recall that $\tern$ is a categorical distribution with three categories $c_1:= ``i\succ j'', c_2:= ``i \cong j''$ and $c_3:= ``i \prec j''$ having probabilities $P_{\succ}^{(i,j)}, P_{\cong}^{(i,j)}$ and $P_{\prec}^{(i,j)}.$
Moreover, by assumption each  $\tern$ has a unique mode, so that we can use Theorem 9 in \cite{jain2022pac}.
As for each run the probability of making an incorrect decision or exceeding $t_0\big(  (P_{\succ}^{(i,j)}, P_{\cong}^{(i,j)},P_{\prec}^{(i,j)}), \nicefrac{\delta}{\binom{n}{2}} \big)$ many samples is bounded by $\nicefrac{\delta}{\binom{n}{2}},$ the probability that the overall sample complexity of \Algo{POCOWISTA} exceeds 
$$
%t(\mathbf{P},\delta) 
%		= \sum\nolimits_{ t=1 }^T t_0\big(  (P_{\succ}^{(i_t,j_t)}, P_{\cong}^{(i_t,j_t)},P_{\prec}^{(i_t,j_t)}), \nicefrac{\delta}{\binom{n}{2}} \big) \leq 
\sum_{i<j} t_0\big(  (P_{\succ}^{(i,j)}, P_{\cong}^{(i,j)},P_{\prec}^{(i,j)}), \nicefrac{\delta}{\binom{n}{2}} \big)$$
is bounded by $\sum_{i<j} \nicefrac{\delta}{\binom{n}{2}} = \delta$ by means of the union bound.

Next, as the modes of the ternary distributions are all correctly identified (with probability  $\nicefrac{\delta}{\binom{n}{2}}$), the score updates are all correct in the sense that
$$		\widehat{CP}(i) \leq CP(i) \leq \overline{CP}(i) \quad \forall i \in \mathcal{A} $$
holds with probability at least $1-\delta.$
Thus, the termination criterion of \Algo{POCOWISTA} implies that
$$	CP(\hat{i}_{\Algo{A}}) \geq \widehat{CP}(\hat{i}_{\Algo{A}}) \geq \overline{CP}(j) \geq CP(j) \, \forall j \in \mathcal{A}\setminus\{\hat{i}_{\Algo{A}}\} 	$$
holds with probability at least $1-\delta,$ so that $\hat{i}_{\Algo{A}}$ is an element of the Copeland set $\mathcal{C}(\mathbf{P}).$
\end{proof}

\clearpage

\section{Proof of Theorem \ref{theorem_tra_pocowista}}\label{sec:proof_theorem_tra_pocowista}

\textbf{Theorem 4.2.}\emph{
Let $\Algo{A}:= \Algo{TRA-POCOWISTA}.$
For any dueling bandits problem with indifferences as in Theorem \ref{theorem_pocowista} which in addition is transitive according to Definition \ref{def:transitivity},  it holds that
\begin{align*}
	\prob\big(\hat{i}_{\Algo{A}} \in \mathcal{C}(\mathbf{P}) \mbox{ and } \tau^{\Algo{A}}(\mathbf{P}) \leq \tilde t(\mathbf{P},\delta) \big) 
	\geq 1- \delta,		
\end{align*}
where 
\begin{center}
	$	\tilde t(\mathbf{P},\delta) 
	= \sum\limits_{t=1}^T t_0\big(  (P_{\succ}^{(i_t,j_t)}, P_{\cong}^{(i_t,j_t)},P_{\prec}^{(i_t,j_t)}), \nicefrac{\delta}{n} \big),$
\end{center}
$t_0$ is as in Theorem \ref{theorem_pocowista} and $T \leq n.$
}
\begin{proof}
	Following the lines of the proof of Theorem \ref{theorem_pocowista}, we only need to verify that \Algo{TRA-POCOWISTA}  runs for at most $n$ many rounds. 
	For this purpose, we only need show that the termination criterion of \Algo{TRA-POCOWISTA} (see line 3 in Algorithm \ref{alg:trapocowista}) is fulfilled after at most $n$ many rounds.
	If the modes of the ternary distributions are all correctly identified (with probability  $\nicefrac{\delta}{n}$) and transitivity as in Definition \ref{def:transitivity} holds, then the score updates are all correct in the sense that
	$$		\widehat{CP}(i) \leq CP(i) \leq \overline{CP}(i) \quad \forall i \in \mathcal{A} $$
	holds with probability at least $1-\delta.$
%	
%	Let us exclude the case, where the Copeland set is a singleton set consisting of one arm with Copeland score $n-1,$ since in light of the latter this arm will 
	
	For sake of convenience, define $\widehat{CP}_t(i)$ as the estimated Copeland score for arm $i \in \mathcal{A}$ \emph{before} the update in round $t$ is made (i.e., the value before line 7 in Algorithm \ref{alg:trapocowista}) and likewise $\overline{CP}_t(i)$.
	The score updates (Algorithm \ref{alg:Trans-Score-Update}) as well as the choice of $j_t$ (line 5 in Algorithm \ref{alg:trapocowista}) imply that
	\begin{align} \label{help_ineq}
			\widehat{CP}_t(j_t) \geq \widehat{CP}_{t-1}(j_{t-1}) + \widehat{CP}_{t-1}(i_{t-1})	+ 1/2	
	\end{align}
	for any round $t.$ 
	Indeed, if one arm dominates the other in round $t-1,$ then it's estimated Copeland score is updated to $\widehat{CP}_{t-1}(j_{t-1}) + \widehat{CP}_{t-1}(i_{t-1})	+ 1,$ while in case of an indifference the updated value corresponds to the right-hand side of \eqref{help_ineq}.
	As $j_t$ is (one of) the arm(s) with largest estimated Copeland score, it has consequently an estimated Copeland score in round $t$ of at least the right-hand side of \eqref{help_ineq}.
	Further, it holds that $ \widehat{CP}_{2}(j_{2}) \geq 1/2.$
	If \Algo{TRA-POCOWISTA} has not terminated after round $n-1$ it must hold that there exists some round $s\in\{1,\ldots,n-1\}$ such that $ \widehat{CP}_{s+1}(i_{s})\geq 1/2,$ as otherwise the ``second arm'' $j_t$ has dominated in each round the ``first arm'' $i_t$ and has stayed the same for all rounds, in which case \Algo{TRA-POCOWISTA} terminates and returns $j_t.$
	Combining this with \eqref{help_ineq} it holds that $	\widehat{CP}_{n+1}(j_n) \geq n/2.$
	We distinguish now the three different cases for the outcome between the compared arms $i_n$ and $j_n$ in round $n$ (i.e., line 6 in Algorithm \ref{alg:trapocowista}) and show that in each case  \Algo{TRA-POCOWISTA} terminates.
	
	\emph{Case 1:} $k=1,$ i.e., $i_n$ dominated $j_n.$
	
	This implies that $\widehat{CP}_{n+1}(i_n) \geq \widehat{CP}_{n+1}(j_n)$ and in particular  $\widehat{CP}_{n+1}(i_n) \geq \widehat{CP}_{n}(i_n).$
	By choice of $i_n$ it holds that for any $i\in \mathcal{A}$
	\begin{align*}
		\overline{CP}_{n+1}(i) \leq \overline{CP}_{n}(i) \leq \overline{CP}_{n}(i_n) = n - 	\widehat{CP}_{n}(i_n) \leq n - 	\widehat{CP}_{n+1}(i_n) \leq n/2 \leq 	\widehat{CP}_{n+1}(i_n).
	\end{align*} 
	Thus, the termination criterion of \Algo{TRA-POCOWISTA} (see line 3 in Algorithm \ref{alg:trapocowista}) is fulfilled after round $n,$  as $i_n$ fulfills the criterion.

	\emph{Case 2:} $k=2,$ i.e., and indifference between $i_n$ and $j_n.$
	
	This implies that $\widehat{CP}_{n+1}(i_n) = \widehat{CP}_{n+1}(j_n)$ and in particular  $\widehat{CP}_{n+1}(i_n) \geq \widehat{CP}_{n}(i_n).$
	By choice of $i_n$ it holds that for any $i\in \mathcal{A}$
	\begin{align*}
		\overline{CP}_{n+1}(i) 
        \leq \overline{CP}_{n}(i) 
        \leq \overline{CP}_{n}(i_n) = n - 	\widehat{CP}_{n}(i_n) 
        &\leq n - 	\widehat{CP}_{n+1}(i_n) \\
        &= n - 	\widehat{CP}_{n+1}(j_n)  
        \leq n/2 \leq 	\widehat{CP}_{n+1}(j_n).
	\end{align*} 
	Thus, the termination criterion of \Algo{TRA-POCOWISTA} (see line 3 in Algorithm \ref{alg:trapocowista}) is fulfilled after round $n,$  as $j_n$ fulfills the criterion.

	\emph{Case 3:} $k=3,$ i.e., $j_n$ dominated $i_n.$
	
	In this case, the potential Copeland score of $i_n$ in round $n+1$ is at most $n - 	\widehat{CP}_{n+1}(j_n).$
	By choice of $i_n$ it holds that for any $i\in \mathcal{A}$
	\begin{align*}
		\overline{CP}_{n+1}(i) \leq \overline{CP}_{n}(i) \leq \overline{CP}_{n}(i_n) \leq n - 	\widehat{CP}_{n+1}(j_n)   \leq n/2 \leq 	\widehat{CP}_{n+1}(j_n).
	\end{align*} 
	Thus, the termination criterion of \Algo{TRA-POCOWISTA} (see line 3 in Algorithm \ref{alg:trapocowista}) is fulfilled after round $n,$  as $j_n$ fulfills the criterion.

\end{proof}
\end{document}